\documentclass{article}

\usepackage{hyperref}

\usepackage[accepted]{icml2020}

\icmltitlerunning{The Intrinsic Robustness of Stochastic Bandits to Strategic Manipulation}

\usepackage{comment}

\usepackage[utf8]{inputenc} %
\usepackage[T1]{fontenc}    %
\usepackage{url}            %
\usepackage{booktabs}       %
\usepackage{amsfonts}       %
\usepackage{microtype}      %

\usepackage{booktabs} %
\usepackage{color}
\usepackage{graphicx}
\usepackage{url}
\usepackage{caption}
\usepackage{subcaption}
\usepackage{multirow}
\usepackage{amsmath,amssymb,amsthm}

\usepackage[font=small,labelfont=bf]{caption}

\usepackage{tikz}
\usetikzlibrary{arrows}
\usetikzlibrary{shapes}
\usetikzlibrary{calc}

\usepackage{paralist}

\newcount\Comments 
\newcommand{\kibitz}[2]{\ifnum\Comments=1{\color{#1}{#2}}\fi}

\newcommand{\new}[1]{\kibitz{black}{#1}}

\newcommand{\E}{\mathbb{E}}

\DeclareMathOperator*{\argmax}{arg\,max}
\newcommand{\1}{\mathbb{I}}

\newcommand{\eps}{\epsilon}
\newcommand{\vareps}{\varepsilon}
\newcommand{\egreedy}{$\varepsilon$-Greedy}
\newcommand{\lsi}{$\mathtt{LSI}$ }

\newcommand{\R}{\mathbb{R}}

\newcommand{\HH}{\mathcal{H}}

\newcommand{\PP}{\mathbb{P}}

\renewcommand{\S}{\mathcal{S}}

\theoremstyle{plain}
\newtheorem{theorem}{Theorem}[section]

\newtheorem{lemma}[theorem]{Lemma}
\newtheorem{definition}[theorem]{Definition}
\newtheorem{proposition}[theorem]{Proposition}

\newtheorem{fact}[theorem]{Fact}

\begin{document}

\twocolumn[
\icmltitle{The Intrinsic Robustness of Stochastic Bandits to Strategic Manipulation}

\icmlsetsymbol{equal}{*}

\begin{icmlauthorlist}
\icmlauthor{Zhe Feng}{harvard}
\icmlauthor{David C. Parkes}{harvard}
\icmlauthor{Haifeng Xu}{virginia}
\end{icmlauthorlist}

\icmlaffiliation{harvard}{Harvard University, Cambridge, Massachusetts, USA}
\icmlaffiliation{virginia}{University of Virginia, Charlottesville, Virginia, USA}

\icmlcorrespondingauthor{Zhe Feng}{zhe\_feng@harvard.edu}

\icmlkeywords{Bandit Algorithm, Strategic Manipulation, Nash Equilibrium}

\vskip 0.3in
]

\printAffiliationsAndNotice{}  %

\begin{abstract}
Motivated by economic applications such as recommender systems, we study the behavior of  stochastic bandits algorithms under \emph{strategic behavior} conducted by rational actors, i.e.,  the arms. Each arm is a \emph{self-interested} strategic player who can modify its own reward whenever pulled, subject to a cross-period budget constraint, in order to maximize its own expected number of times of being pulled. We analyze the robustness of three popular  bandit algorithms: UCB, $\varepsilon$-Greedy, and Thompson Sampling. We prove that all three algorithms achieve a regret upper bound $\mathcal{O}(\max \{ B, K\ln T\})$  where $B$ is the total budget across arms, $K$ is the total number of arms and $T$ is length of the time horizon.  This regret guarantee holds under \emph{arbitrary  adaptive} manipulation strategy of  arms.  Our second set of main results shows that this regret bound is \emph{tight}--- in fact for UCB it is tight even when we restrict the arms' manipulation strategies to form a \emph{Nash equilibrium}.  The lower bound makes use of a simple manipulation strategy, the same for all three algorithms, yielding a bound of $\Omega(\max \{ B, K\ln T\})$. Our results illustrate the robustness of classic bandits algorithms against strategic manipulations as long as $B=o(T)$. 
\end{abstract}

\section{Introduction}

Multi-armed bandits (MAB) algorithms play a significant role in learning to make decisions across the digital economy, for example in online advertising \cite{CMR15,FSS19}, search engines \cite{KSWA15}, and recommender systems \cite{LCLS10}. Classical stochastic MAB models assume that the reward feedback of each arm is drawn from a fixed distribution. However, in many economic applications,  an arm may be \emph{strategic} and able to modulate its own reward feedback in order to further its own objective, e.g., increasing the number of times it is selected. For instance,  restaurants may  offer discounts or free dishes in order to entice customers to return, and  sellers on Amazon may offer discounts or coupons in order to receive higher ratings and thus increase their ranking.  %

We distinguish two different kinds of actors in our  strategic setting: the \emph{principal} and the \emph{arms}. 
The principal represents a multi-armed bandit algorithm,
corresponding to a system, such as the Amazon marketplace platform.
The arms  represent the parties who generate reward feedback to the principal,
for example the  sellers on Amazon.
We assume that the {\em true reward} of each arm is drawn from an underlying distribution. Further, we model each  arm $i$ as a  strategic agent, able to  manipulate its own reward, but subject to a total budget $B_i$ across all time periods.
The objective of an arm is to maximize its expected number of times being pulled. Arms can only modify their own reward feedback, and have no control over the rewards of the other arms. An arm's strategy can be adaptive--- that is, the amount by which an arm modulates the current reward can depend on his own history  of realized
rewards and manipulations. %
Since arms' strategies affect each other, through the MAB algorithm, this dynamic interaction forms a situation of strategic interdependence among arms, more precisely, a {\em stochastic game}.

This study is motivated by various economic applications of MAB, where strategic manipulations appear more realistic than the more conservative consideration of  {\em adversarial attacks}~\cite{JLMZ18,LMP18}.
The central question that we study in this paper is the following:

\begin{quote}
 \em{Are existing stochastic bandit algorithms robust to strategic manipulation by arms? Quantitatively, can we characterize their regret bounds? } %
\end{quote}
\vspace{-10pt}

For a motivating example, suppose that  a recommender system such as Yelp runs a stochastic bandit algorithm \new{to recommend a single restaurant to each user}. The arms correspond to restaurants to be recommended and each user access to the system corresponds to a pull of the arms. %
The true service quality of each restaurant follows some underlying distribution. However, restaurants are strategic, and a natural objective is to maximize the expected number of times a restaurant is recommended to users. To do so, it is common to provide discounts to some user (modified rewards in our model), subject to budget constraints because the restaurants cannot provide arbitrarily many discounts. 
In this context, our goal is to understand how the strategic behavior of restaurants can affect the platform's regret.

\subsection{Our Results and Implications}

\noindent {\bf Results. } Our main results illustrate that the three popular stochastic bandits algorithms of Upper Confidence Bound (UCB),
$\eps$-Greedy, and Thompson Sampling, are robust to strategic manipulations. %
Specifically, we show that the regret of all three algorithms is upper bounded by $\mathcal{O} \big( \sum_{i \not = i^*} \max \{ B_i, \frac{\ln  T}{\Delta_i}  \} \big)$, where $i^*$ indexes the optimal arm w.r.t. the true rewards, and $\Delta_i$ is the difference in the mean of the true reward between arm $i$ and $i^*$.
\new{For convenience, we assume throughout the paper that $B_{i^*} = 0$, since any $B_{i^*} >0$ would only help $i^*$ to be pulled more, and thus benefit the principal.}   %
Interestingly, the regret bound holds for arbitrary adaptive arm strategies.

One natural question is whether it is possible to achieve smaller regret bounds if we restrict strategies to form a \emph{Nash equilibrium}, which is the standard solution concept in game theory. We answer this question in the negative,  at least for UCB. %
We characterize the dominant-strategy equilibrium of the game
induced by the UCB algorithm,  and prove a lower bound on regret of $\Omega(\max \{ B, K\ln T  \})$ for equilibrium arm manipulations, where $K$ is the number of arms and $B$ is the total budget across arms. 
This shows that the upper bound is essentially tight, even under equilibrium behaviors.  All our bounds hold for both bounded and unbounded rewards.
We also provide a matching lower bound for \egreedy{} and TS under a natural, lump sum investing strategy, in which an arm spends all of its budget the first time it is pulled.  We have not been able to show whether or not this strategy forms a Nash equilibrium in the induced stochastic game, and leave open the question of whether the regret bound is also tight for \egreedy{} and Thompson sampling (TS) under equilibrium behavior. 

\noindent {\bf Implications. }  %
These results show  that the performances of all three MAB algorithms  deteriorates linearly in the total budget $B = \sum_{i \not = i^*} B_i$.  
As long as $B = o(T)$, the optimal arm will  be pulled $T - o(T)$ times. %
The simulation results also validate this linear dependence on $B$.

Since our upper bounds on regret hold for arbitrary arm behaviors, even allowing for 
reducing the reward on arms,
they can also  correspond to the choices of a single adversary, %
and the  results also shed light on \emph{adversarial attacks} on stochastic bandit algorithms.  
In contrast to existing adversarial models, the key difference is that the reward of the optimal arm, $i^*$, cannot be modified.
With rational behavior, this is without loss; if the optimal arm had an associated budget then this can only lead to more  pulls of  this  arm and lower regret.
Our results show that if a single adversary cannot contaminate the optimal arm, then standard bandits algorithms are already robust.  The bound would also  hold in a more general setting
in which the optimal arm's reward can only be increased.

Concretely, the results can be alternatively interpreted as follows:
for an adversarial corruption model that is modified to prevent contamination of the optimal arm, then UCB, $\varepsilon$-Greedy,  and TS all have regret $O(\max \{ B, K\ln T \})$, and are robust as long as $B = o(T)$. This is in sharp contrast to the  situation of \emph{unrestricted adversarial attacks},  where an attack budget of $\mathcal{O}(\ln  T)$ can  lead algorithms such as UCB and $\varepsilon$-Greedy to suffer regret $\Omega(T)$ \cite{JLMZ18,LMP18}. Even for state-of-the-art, robust bandits algorithms~\cite{Gupta2019colt}, the regret bound $O(KB + \sum_{i \not = i^*} \frac{\log T}{\Delta_i}\log(\frac{K}{\delta} \log T))$   is worse than the bound in the present paper by a factor of $K$ (when $B = \Omega(\log T)$).

Another implication of the present work is to the problem of \emph{incentivizing exploration}, where the principal relies on users to pull arms \cite{Frazier2014incentivizing,Wang2018multi}, and users are modeled as myopic and only care about their immediate reward. The idea is that the principal can provide rewards to encourage more exploration. At the same time, it has been observed in field experiments that users are generally biased towards reporting a higher evaluation when provided with these kinds of incentives, i.e., an upwards-biased reward.  Our results have been applied by~\citet{LiuAAAI20} to show that bandit algorithms are robust to this kind of bias: if
reported rewards can only be upwards-biased (a special case of our model),
then the bandit algorithm will be robust, also allowing for the reward feedback on
the optimal arm to be affected.

\subsection{Additional Related Work}  

In this work, we study strategic manipulation in the context of classical  stochastic bandit algorithms. This is similar in spirit to~\citet{JLMZ18}, who study  adversarial attacks to UCB and  $\epsilon$-Greedy.  The relation and differences between their results and ours are elaborated above. Another related, and complementary, line of research is on designing new algorithms for stochastic bandits that are robust to adversarial corruptions \cite{LMP18,Gupta2019colt}.  In principle, we could have also studied these algorithms in the present context. However, we believe that it remains important to understand the  conditions under which classical, simple bandit algorithms work well,  because they are likely to be used in real-world applications. Moreover,
the regret guarantees
 of these classical algorithms,
 in our strategic setup, is better than the bounds available for these robust algorithms
 under adversarial corruptions. It is an interesting open question to understand whether these robust algorithms can achieve the same or even better regret bound when restricted to our strategic setup. %
 Another further work is to understand strategic behavios in the recent line of works in non-stationary bandits, e.g.,~\citet{BGZ19, WSZ19}.

This  work belongs to the general field of %
\emph{no-regret learning with strategic agents}. Much of this literature is focused on designing no-regret learning algorithms under  strategic behavior, and has studied problems arising from concrete applications such as auctions, e.g., \cite{BKRW04,WRP16,Feldman2016,FPS18} and recommender systems \cite{Bayesian15,Bayesian19}. 
However, the strategic behavior in these models do not correspond to arm manipulation, but rather correspond to bidding strategies or auction mechanisms.   
To our knowledge,  \citet{BMSW17} are the first to consider  strategic behaviors of arms in stochastic bandit settings. In their model, when an arm is pulled, it  receives a private reward $v$ and strategically chooses an amount $x$ to pass to the principal, leaving the remaining amount of  $v-x$ to the arm itself. Motivated by a different application context, our model considers strategic arms that seek to maximize their expected number of plays by manipulating their reward feedback under a budget.  

\section{The Model: Strategic Manipulations in Stochastic Bandits}\label{sec:model}

We consider a strategic variant of the  stochastic multi-armed bandit problem. There are $K$ arms, denoted by $[K] = \{ 1,2,...,K \}$. The reward of each arm $i \in [K]$ follows a $\sigma$-sub-Gaussian distribution (see Definition~\ref{def:sub-gaussian} 
in Appendix)
with mean $\mu_i$, where parameter $\sigma$ is publicly known. The $\sigma$-sub-Gaussian assumption is widely used in MAB literature~\cite{BCB2012}.
Let $i^*= \argmax_{i\in [K]} \mu_i$ denote the unique arm (WLOG) with maximum mean, $\Delta_i= \mu_{i^*} - \mu_i >0$ denote the difference of the reward mean between the optimal arm $i^*$ and arm $i$ ($\not = i^*$), and $\underline{\Delta} := \min_{i\neq i^*}\Delta_i$. %

There are two different parties: the principal and the arms. The principal represents a bandit algorithm, in particular, UCB, $\epsilon$-Greedy, or TS.  At each time $t = 1, \cdots, T$, the principal pulls arm $I_t$, which generates a reward $r_t$. Here $T$ is some fixed time horizon. Let $n_{i}(t) = \sum_{\tau = 1}^t \1(I_t = i)$ denote the number of times that arm $i$ has been pulled up to and including time $t$, and $\widehat{\mu}_{i}(t) = \frac{1}{n_i(t)} \sum_{\tau = 1}^t r_\tau \cdot \1(I_\tau=i)$ denote the average rewards obtained from pulling arm $i$ up to and including time $t$.

Each arm $i \in [K]$ is a strategic actor, equipped with the objective of maximizing $\E [n_i(T)]$, i.e., the expected total number of times it is pulled. This is a natural objective
in systems such as recommender systems.

The actions available to arm $i$ is to modify its reward feedback when pulled, subject to a total budget $B_i$ across rounds.  Concretely, when $I_t = i$, arm $i$ can add an additional reward amount $\alpha_t^{(i)}$ to the realized reward $r_t$,\footnote{%
In this paper, $\alpha_t^{(i)}$ can be negative, if that helps $i$. None of our results  rely on the positivity of $\alpha_t^{(i)}$'s.} subject to budget constraint $\sum_{t = 1}^T |\alpha_t^{(i)} | \leq B_i$,  so that the revealed reward to the principal is $\widetilde{r_t} = r_t + \alpha_t^{(i)}$. We refer to $r_t$ as the \emph{true reward} and $\widetilde{r_t}$ the \emph{manipulated reward}.   \new{The \emph{adaptive}  manipulation strategy of arm $i$ is a function $S^{(i)}: \HH^{(i)}_{t-1}\times [K]\rightarrow \R$, mapping its own up-to-$t$ history $h^{(i)}_{t-1}\in \HH^{(i)}_{t-1}$ and $I_t$ to a manipulation $\alpha^{(i)}_t$. The history $h^{(i)}_{t} = \{I_{\tau}, r_\tau, \alpha^{(i)}_\tau\}_{\tau: I_\tau=i, \tau\leq t}$ is the information that arm $i$ observed up to time $t$, which includes the pulling history, realized rewards, and manipulations of arm $i$ at past $t$ rounds. Let $h_t = \{ h_t^{(i)} \}_{i \in [n]}$ denote the histories of all arms until time $t$.  Arm $i$ has no access to the information of the other arms, hence the strategy only takes his own historical information as input.  We use $S^{(-i)}$ to define the strategies of the other arms. 
Given a history $h^{(i)}_t$, the remaining budget and $n_i(t)$ are determined. }

Arm $i$ has no control over other arms' rewards. Therefore, $\alpha^{(i)}_t$ must equal $0$ for $I_t \not = i$ and any history $h^{(i)}_{t-1}$.  
\new{For convenience of the analysis, we assume  $B_{i^*} = 0$ throughout the paper and  thus $\alpha^{(i^*)}_t = 0$ for any $t$,  since any reasonable $\S^{(i^*)}$ with $B_{i^*}>0$ would only lead to more pulls of $i^*$  and thus benefit the principal.  Let 
\begin{eqnarray*}
\beta_t^{(i)}(h^{(i)}_{t-1}, I_t) = \sum_{\tau \leq t} S^{(i)} (h^{(i)}_{\tau-1}, I_\tau)
\end{eqnarray*} 
denote the total manipulation by arm $i$ until time $t$ with manipulation strategy $S^{(i)}$ and a realized history $h^{(i)}$, which satisfies $\beta_T^{(i)}(h, I_T) \leq B_i, \forall h\in \HH^{(i)}_{T-1}$ and $I_T\in [K]$. %
When the history $h^{(i)}_{t-1}$ and selected arm $I_t$ are clear from the context, we sometimes omit this and write $\beta_t^{(i)}$ for notational convenience.

The objective of arm $i$ is to find a strategy $S^{(i)}$ to maximize $\E [ \sum_{t=1}^T \1\{I_t = i\}]$,\footnote{Throughout the paper, the expectation is over all the randomness in algorithms and the rewards.} by manipulating its reward to trick the principal to pull arm $i$ more.}
The principal observes only $\widetilde{r_t}$ and not true reward $r_t$. The goal of the principal is to minimize regret with respect to  the true reward $r_t$. This is without loss of generality since the aggregated reward with respect to   $\widetilde{r_t}$ differs from the true reward by at most the total manipulation budget $B = \sum_i B_i$,  which is the same order as our regret bounds.

{\bf \lsi manipulation. } A particular manipulation strategy that will be of interest is the  \emph{Lump Sum Investing} ($\mathtt{LSI}$) strategy, in which an arm simply spends all of its remaining budget whenever first pulled. \new{For arm $i$, the \lsi is a strategy $S^{(i)}$ that at any time $t$ and any history $h^{(i)}_{t-1}\in \HH^{(i)}_{t-1},$ $S^{(i)}(h^{(i)}_{t-1}, I_t) = B_i - \sum_{\tau=1}^{t-1}\alpha^{(i)}_\tau$ when $I_t = i$.}

\subsection{Solution Concepts }\label{sec:solution-concept}

This is a situation of strategic interaction, where the MAB algorithms induce a  stochastic game. Our main goal  is to  quantify the principal's regret in this game, as measured with respect to the true reward.
Despite the widely-known intractability in characterizing Nash Equilibria for general stochastic games  \cite{Ben1990complexity,conitzer2003complexity}, \new{we show that when the principal runs UCB, there is a \emph{subgame perfect Nash equilibrium} (SPE) in our game, where each arm simply plays the \lsi strategy.   %
A strategy profile  $S^*  = (S^{*(1)},\cdots,S^{*(K)})$ is a SPE if $S^{*(i)}$ is an optimal strategy for any arm $i$, given any history $h_{t-1}$, and given the strategies $S^{(-i)}$ of the other arms, for any $t$. %
In fact, we show that \lsi{} is   \emph{dominant strategy} when the principal runs the UCB algorithm, that is \lsi is an optimal strategy for arm $i$ for any $t$, given any history $h_{t-1}$, and whatever the strategies $S^{(-i)}$ of the other arms.}
This provides a very strong suggestion as to the
kind of behavior we should expect from arms. The upper bounds on regret hold  for arbitrary adaptive manipulations, regardless whether they form a SPE or not.
The matching lower bounds on regret for UCB are proved under the dominant-strategy SPE.  Not only does this show that the upper bounds are tight, but it highlights the special role of the SPE in this UCB setting. %

\section{UCB is Robust to Strategic Manipulations}\label{sec:ucb-greedy}

In this section, we provide a regret analysis for the Upper Confidence Bound (UCB) principal in our strategic setup. %
We first show an upper bound on the regret for arbitrary arm strategies. Next, we prove that this regret bound is tight even under equilibrium arm behaviors. Finally, we discuss how to generalize the results to the bounded reward setting.  
The formal proofs 
can be found in Appendix~\ref{app:ucb-greedy}. 
\vspace{-5pt}
\subsection{Regret Upper Bound for UCB Principal}\label{sec:ucb-upper-bound}
We consider a standard  $(\alpha, \psi)-$UCB with $\alpha=4.5$, $\psi: \lambda \rightarrow \frac{\sigma^2 \lambda^2}{2}$ and thus $(\psi^*)^{-1}(\eps)=\sqrt{2\sigma^2\eps}$
\cite{BCB2012} . Concretely, the algorithm selects each arm once in the first $K$ rounds, i.e. $I_t = t, \forall t < K$. For $t\geq K$,
\vspace{-5pt}
\begin{footnotesize}
\begin{eqnarray*}
	I_t =\argmax_i\left\{\widehat{\mu}_{i}(t-1) + 3\sigma\sqrt{\frac{\ln T }{n_{i}(t-1)}}  + \frac{\beta^{(i)}_{t-1}}{n_i(t-1)}\right\}, 
\end{eqnarray*}
\end{footnotesize}
where $\beta^{(i)}_{t-1}$ is the aggregated manipulation of arm $i$ up to (including) $t-1$. The term  $\widehat{\mu}_{i}(t-1) + 3\sigma\sqrt{\frac{\ln T }{n_{i}(t-1)}}$ is the standard \emph{UCB term}\footnote{\new{There is also a UCB variant that uses time-dependent confidence width $3\sigma\sqrt{\frac{\ln t}{n_{i}(t-1)}}$. Both versions are common in the literature.
Our regret upper bound holds for both, but it appears that the $(\ln T)$ version is more convenient for  the analysis of lower bounds in equilibrium.}} 
for any arm $i \in [K]$ at time $t$, which we denote as $\mathrm{UCB}_i(t)$. Let $\widetilde{ \mathrm{UCB}}_i(t)  = \mathrm{UCB}_i(t) + \beta^{(i)}_{t-1} \big/ n_i(t-1)$ represent the \emph{modified UCB term} for the strategic arm $i \, (i\neq i^*)$ with manipulation strategy $S^{(i)}$ (recall $\beta^{(i)}_t $ is induced by $S^{(i)}$, and $\beta^{(i^*)}_t = 0$ always).

The main result in  this section is an upper bound for regret $\E[R(T)]$ under an arbitrary adaptive manipulation strategy $S$.  
\begin{theorem}\label{thm:ucb-ub}
	For any manipulation strategy $S$ of the strategic arms, the regret of the UCB principal is bounded by
	\begin{align*}
	\E[R(T)] \leq \sum_{i\neq i^*} \big[\max\left\{3B_i, \frac{81\sigma^2\ln T}{\Delta_i} \right\} +\big(1 + \new{3}\Delta_i \big]
	\end{align*}
\end{theorem}
\vspace{-10pt}
Theorem~\ref{thm:ucb-ub} reveals that the UCB algorithm is robust in our strategic model of arm manipulations. If the budget of each arm is bounded by $\mathcal{O}(\ln T)$, the regret of the principal is still bounded by $\mathcal{O}(\ln T)$. If $B_i = \Omega (\ln T)$ for some arm $i$'s, the regret is upper bounded by $\mathcal{O}(\sum_{i \not = i^*} B_i)$. This is sublinear in $T$ as long as $B = \sum_{i \not = i^*} B_i= o(T)$.

Theorem \ref{thm:ucb-ub}  strictly generalizes the regret bound of the standard UCB framework, which corresponds to a special case with no budgets.
Fixing any manipulation strategy $S$, the proof starts by re-writing the regret  in the following format: %
\begin{eqnarray}\label{eq:regret-decomposition}
\E[R(T)] = \sum_{i\neq i^*} \Delta_i \cdot \E[n_i^{S}(T)]. 
\end{eqnarray}
\vspace{-15pt}

What remains is to bound $\E[n_i^{S}(T)]$ for each arm $i$.  For convenience, we omit the superscript $S$ since it is clear that we focus on an arbitrary $S$.
Lemma~\ref{lem:ucb-ub}
gives the upper bound of $\E[n_i(T)]$ for each arm $i$,
and combined with~\eqref{eq:regret-decomposition},  yields a proof of Theorem~\ref{thm:ucb-ub}.   
\begin{lemma}\label{lem:ucb-ub}
Suppose the principal runs UCB. For any manipulation strategy $S$ of strategic arms, the expected number of times that arm $i (i\neq i^*)$ is pulled up to time $T$ can be bounded as follows,
\begin{align*}
\E[n_i(T)] \leq \max\left\{\frac{3B_i}{\Delta_i}, \frac{81\sigma^2\ln T}{\Delta_i^2} \right\} + \new{3}
\end{align*}
\end{lemma}
\vspace{-15pt}

\begin{proof}[Proof Sketch]
  The main difference from the analysis of the standard UCB is to choose a proper threshold $C_i(T)$ for $n_i(t-1)$ so that we can have the best trade-off between the two terms in the following decomposition of $\E[n_i(T)]$:
  
\begin{small}
\begin{align*}
\E\left[n_i(T)\right] \leq 1+ \E\left[\sum_{t=K+1}^T \1\{I_t=i, n_i(t-1)\leq C_i(T)\}\right] \\
+ \E\left[\sum_{t=K+1}^T \1\{I_t=i, n_i(t-1)\geq C_i(T)\}\right].
\end{align*}
\end{small}
After careful manipulation, it turns out that  $C_i(T) = \max\left\{\frac{81\sigma^2\ln T}{\Delta_i^2}, \frac{3B_i}{\Delta_i} \right\}$ gives the correct regret bound, after bounding  the first term directly by $C_i(T)$ and bounding the second term via the  Chernoff-Hoeffding inequality.  The formal proof is shown in Appendix~\ref{app:ucb}.
\end{proof}

\subsection{Tightness of the Regret Bounds at Equilibrium}\label{sec:ucb-greedy-tight}

The above regret bound for UCB  holds for arbitrary adaptive manipulation strategies. This raises the following question:  \emph{is it possible to achieve better regret upper bounds by restricting arm manipulations to form a \new{subgame perfect} Nash equilibrium?} 
We provide a negative answer to this question, and prove that the regret upper bounds are tight even  in equilibrium.
We  first prove that \lsi{} is a \emph{dominant strategy} for each arm \new{in any subgame} --- an optimal strategy regardless of what strategies other arms use, \new{given any realized history $h_{t-1}$} --- when the principal runs UCB. As a consequence, each arm playing \lsi forms a dominant-strategy SPE. We then establish a lower bound on regret when each arm plays the \lsi strategy, and show
that this bound matches the upper bound. %

Concretely, we first prove that the (random) number of times that arm $i$ is pulled  under strategy $\mathtt{LSI}$ first-order \emph{stochastically dominates}
the number of times pulled under any other adaptive manipulation strategy $S^{(i)}$, given any fixed history.  
\begin{theorem}\label{thm:opt-attack}
Suppose $T\geq K$, and the principal runs the  UCB algorithm. For any arm $i$, any strategy $S^{(i)}$, and any strategy profile $S^{(-i)}$ of others, and for any time $t$ and history $h_{t-1}^{(i)}$, we have
\begin{equation}\label{eq:LSI-optimal}
\begin{aligned}
\PP[n_{i}^{(\mathtt{LSI},  S^{(-i)})}(t:T) \geq n] \geq \PP[n_{i}^{S}(t:T) \geq n], \, \,  \forall  n\in \mathbb{N}, 
\end{aligned} 
\end{equation}
where $n_{i}(t:T) = \sum_{\tau=t}^T \1\{I_\tau = i\} $ is the total number of pulls of arm $i$ from $t$ to $T$. That is,	$n_{i}^{ (\mathtt{LSI}, S^{(-i)}) }(t:T)$ 	  \emph{first-order stochastically dominates} $n_{i}^{S}(t:T)$. Therefore, $ \E [ n_{i}^{(\mathtt{LSI}, S^{(-i)})}(t:T)] \geq \E [n_{i}^{S}(t:T)]$, and  thus \lsi{}is a best response to any $S^{(-i)}$. 

\end{theorem}

It follows directly from Theorem \ref{thm:opt-attack} that each arm playing \lsi{}forms a dominant-strategy SPE. The complete proof of Theorem~\ref{thm:opt-attack} is quite involved, and can be found in Appendix~\ref{app:lsi-optimality}.

To see why this conclusion is not obvious, let us illustrate the trade-off in designing the optimal manipulation strategy.  %
The advantage of the \lsi{}strategy in UCB is to significantly increase the arm's UCB term and receive many pulls at the very beginning.  This, however, also comes with a disadvantage--- it quickly decreases the confidence width (the $3 \delta \sqrt{\ln T}/\sqrt{ n_i(t-1)}$ term) and the effect of the manipulation (the $\beta_{t-1}^i/ n_i(t-1) $ term) in the UCB term, whereas other arms' confidence width and manipulation effect remain large.  For this reason, it may also be beneficial for an arm to defer its manipulation to later rounds so that it avoids fierce competition in the early few rounds resulting from other arms' large confidence width, large manipulation effect, and possibly large rewards due to lucky draws.

The proof shows that in this intricate random process, the aforementioned advantage of using \lsi{} always dominates its disadvantage.  We make use of the  \emph{coupling technique}~\cite{Thorisson2000} to compare the random sequence of pulled arms  when arm $i$ uses $\mathtt{LSI}$ compared with an arbitrary strategy $S^{(i)}$. A crucial step  is to show that under coupling of the two stochastic processes, either $\mathtt{LSI}$ results in more pulls of arm $i$ than $S^{(i)}$ or they must result in \new{each of the other arms to be pulled for the same number of times.} %
We then argue that in the latter case, $\mathtt{LSI}$ must also be better than $S^{(i)}$ because they face the same outside competition but the modified UCB term of $\mathtt{LSI}$ is larger than the modified UCB term of $S^{(i)}$. As a consequence, $\mathtt{LSI}$ performs better than $S^{(i)}$ in both cases, yielding a proof of the theorem.

To show that the regret bounds in Section \ref{sec:ucb-upper-bound} are tight, it will suffice to develop a lower bound on regret for when each arm plays \lsi{}, as shown in the following theorem. %

\begin{theorem}[Regret Lower Bound at Equilibrium]\label{thm:ucb-egreedy-lb}
Suppose the principal uses UCB algorithm and each arm uses $\mathtt{LSI}$. For any $\sigma$-sub-Gaussian reward distributions on arms, the regret of the principal satisfies,
\vspace{-5pt}
\begin{align*}
\E[R(T)] \geq \underline{\Delta}\sum_{i\neq i^*} \frac{B_i}{2\Delta_i} - \mathcal{O}\left(\frac{\ln T}{\underline{\Delta}}\right). 
\end{align*}
\vspace{-10pt}
\end{theorem}

The proof of Theorem \ref{thm:ucb-egreedy-lb} differs from standard techniques in proving regret lower bounds, and is carefully tailored to achieve tight bounds with respect to budget $B_i$'s. Classical regret lower bounds are typically proved by constructing a particular class of distributions, i.e., Bernoulli \cite{BCB2012}, and then arguing that the given algorithm cannot do very well on these constructed instances. These bounds are usually distribution-dependent. Our proof takes a completely different route. Indeed, our technique results in a lower bound that holds for arbitrary $\sigma$-sub-Gaussian distributions and thus is distribution-independent. %

The proof of Theorem \ref{thm:ucb-egreedy-lb}  starts with a simple lower bound for the regret $\E[R(T)]$ by utilizing Equation~\eqref{eq:regret-decomposition}: 
\begin{eqnarray}\label{eq:regret-lb}
\E[R(T)] = \sum_{i\neq i^*} \Delta_i \E[n_i(T)] \geq \underline{\Delta}\cdot \sum_{i\neq i^*} \E[n_i(T)]. 
\end{eqnarray}
We then only need to focus on lower bounding $\sum_{i\neq i^*}\E[n_i(T)]$ when all the arms play strategy $\mathtt{LSI}$. %
We  prove an upper bound for $\E[n_{i^*}(T)]$, which  translates to a lower bound for $\sum_{i\neq i^*}\E[n_i(T)]$. However, upper bounding $\E[n_{i^*}(T)]$ requires quite different techniques than upper bounding $\E[n_{i}(T)]$ for any non-optimal arm $i$. %
A crucial step  is to argue that when $i^*$ has been pulled more than $C$ times (for some carefully chosen threshold $C$), it will become much less likely to be pulled again.  This differs from standard techniques for upper bounding $\E[n_{i}(T)]$ for non-optimal arm $i$, for two reasons: (1) we have to compare the UCB term of arm $i^*$ with all the other non-optimal arms' UCB terms, whereas to upper bound $\E[n_{i}(T)]$, one typically compares $i$ with only the optimal arm $i^*$; (2) we need to argue $i^*$ is pulled with small probability despite $\mu_{i^*} > \mu_i$ whereas upper bounding $\E[n_{i}]$ is more natural when $\mu_{i^*} > \mu_i$.  To overcome these challenges,  we carefully decompose the $\E[n_{i^*}(T)]$ term and pick thresholds not only for $n_{i^*}(t-1)$, but also for $n_i(t-1)$ for each non-optimal arm $i\neq i^*$. %
A complete proof of Theorem \ref{thm:ucb-egreedy-lb} can be found in  Appendix~\ref{app:ucb-egreedy-lb}.

\noindent {\bf Remarks:} 
The lower bound holds for arbitrary $\sigma$-Gaussian distributions, and may be negative in value, and thus not meaningful when $B_i = o(\ln T)$.  However, the bound can be easily converted to a distribution-dependent lower bound $\max \big\{ \underline{\Delta}\sum_{i\neq i^*}\frac{B_i}{2\Delta_i} - \mathcal{O}\big(\frac{\ln T}{\underline{\Delta}}\big),  \,\Omega\left(K\ln T\right) \big\}$  because there exist distributions such that any no-regret learning algorithm will suffer regret  $\Omega\left(K\ln T\right)$ \cite{BCB2012}   and the non-optimal arms' manipulation strategy would only increase  the regret. This distribution-dependent lower bound  precisely matches %
the upper bound $\mathcal{O} \big(  \max \{ B, K\ln T \} \big) $ in Section \ref{sec:ucb-upper-bound}. 

\subsection{Generalization to Bounded Rewards}\label{sec:bounded-rewards}

In many applications, such as where the rewards are ratings provided by customers on platforms such as those operated by Yelp and Amazon, the rewards are bounded within some known interval (e.g. 0 $\sim$ 5 stars rating).
Suppose, for example, that the reward is bounded within $[0,1]$. %
In such settings, the $\mathtt{LSI}$ strategy may be infeasible since the strategic arm can increase its reward to at most the upper bound.  In this case,  arms can use a natural variant of $\mathtt{LSI}$ for bounded rewards: each arm $i$ spends its budget to promote the realized reward to the maximum limit of $1$ whenever it is pulled, and does so until it runs out of budget $B_i$. We term this  natural variant
the {\em Lump Sum Investment for Bounded Rewards strategy},
or $\mathtt{LSIBR}$ for short. 

Theorem \ref{thm:opt-attack} can be easily generalized to this bounded reward setting. Each arm playing $\mathtt{LSIBR}$  forms a dominant-strategy subgame perfect Nash equilibrium in the bounded reward setting.
The more challenging task is to prove a similar lower bound on regret. To do so, we provide a unified reduction from any regret lower bound under $\mathtt{LSI}$ to a regret lower bound under $\mathtt{LSIBR}$, with an additional  loss of $\Theta (\ln T)$. Our reduction applies to \new{any stochastic bandit algorithms.}%
The main findings are summarized in Theorem~\ref{thm:br-lb}.
\begin{theorem}\label{thm:br-lb}

For any stochastic bandit algorithm, let $\E\left[R^{\mathtt{LSI}}(T)\right] $ (resp. $\E\left[R^{\mathtt{LSIBR}}(T)\right] $ ) denote the regret in the unbounded (resp. bounded) reward setting, where each arm uses $\mathtt{LSI}$ (resp. $\mathtt{LSIBR}(T)$). We have 
	\begin{align*}
\E\left[R^{\mathtt{LSIBR}}(T)\right]  \geq \E\left[R^{\mathtt{LSI}}(T)\right]  - O(\frac{\underline{\Delta}\ln T}{(1-\mu_{i^*})^2}) 
\end{align*}
	\vspace{-18pt}
\end{theorem}

\section{The Robustness of \egreedy{} and Thompson Sampling}\label{sec:ts}

\new{In this section, we  turn our attention to two other popular classes of MAB algorithms, i.e., \egreedy{} and Thompson Sampling (TS) \cite{TS33,AG2017}. Unlike UCB, these are  \emph{randomized} algorithms: \egreedy{} algorithm involves a random exploration phase and TS employs random sampling during arm selection (note: the randomness when executing UCB comes purely from the random rewards and not the algorithm itself).  We establish the same regret upper bound for  \egreedy{} and Thompson Sampling, again for arbitrary adaptive manipulation strategies. However, the additional randomness involved in \egreedy{}  and TS makes it much more challenging to exactly characterize the SPE in the induced games. Nevertheless, we show that the regret upper bounds remain tight  under the \lsi strategy. 
}
\vspace{-10pt}

\subsection{Regret Upper Bound for $\vareps$-Greedy Principal}\label{sec:e-greedy}

As with UCB, we assume that the algorithm pulls arm $t$ when $t \leq K$, i.e., first exploring each arm once. At round $t > K$, the algorithm selects an arm as follows: 
$$
I_t = \left\{
\begin{aligned}
&\text{uniformly drawn from }[K], &\text{w.p. } \varepsilon_t\\
&\argmax_i \left\{ \widehat{\mu}_{i}(t-1) + \frac{\beta^{(i)}_{t-1}}{n_i(t-1)}  \right \} , &\text{o.w.}
\end{aligned}
\right.
$$
\vspace{-10pt}

The first step above is \emph{Exploration}, while the second step is  \emph{Exploitation}. We choose $\vareps_t=\Theta\left(\frac{1}{t}\right)$, which guarantees the convergence of the algorithm~\cite{ACBFS02}. 
We prove the following regret bound for $\vareps$-Greedy, again for an arbitrary adaptive manipulation strategy $S$.  As with the UCB case, the result strictly generalizes previous analysis for $\vareps$-Greedy to incorporate the effect of manipulations. 
\begin{theorem}\label{thm:e-greedy-ub}
For any adaptive manipulation strategy $S$ of strategic arms, the regret of the $\vareps$-Greedy principal
with  $\eps_t = \min\{1, \frac{cK}{t}\}$ and $c = \max\{20, \frac{36\sigma^2}{\underline{\Delta}}\}$,
is bounded by
\begin{align*}
\E[R(T)] \leq \sum_{i \neq i^*}\big[3B_i  + \mathcal{O}\left(\frac{\ln T}{\Delta_i}\right)\big].
\end{align*}
\end{theorem}

\subsection{Regret Upper Bound for Thompson Sampling Principal}

We model rewards with Gaussian priors and likelihood.
As with UCB and $\varepsilon$-Greedy, we also assume that the algorithm %
pulls each arm once in the first $K$ rounds. At round $t> K$, the algorithm selects an arm %
according to the following procedure:
\vspace{-10pt}
\begin{itemize}
\item[(1)] For each $i\in [K]$, sample $\theta_{i}(t-1)$ from a Gaussian distribution $\mathcal{N}(\widetilde{\mu}_i(t-1) , \frac{1}{n_i(t-1)})$, where $\widetilde{\mu}_i(t-1) =  \widehat{\mu}_i(t-1) + \frac{\beta^{(i)}_{t-1}}{n_i(t-1)}$. %
\item[(2)] Select arm $I_t = \argmax_i \theta_i(t-1)$.
\end{itemize}
\vspace{-10pt}

The total manipulation  by arm $i$ until time $t$, $\beta^{(i)}_t$, is induced by a strategy profile $S$.
TS is widely known to be challenging to analyze, and its regret bound was proved only recently~\cite{AG2017}.  This is because the algorithm does not directly depend on the empirical mean of each arm, but relies on random samples from the prior distribution centered at the empirical mean. This sampling process further complicates the analysis of the stochasticity in the algorithm.  Moreover, it is unclear whether there exists an effective adversarial attack to TS. This was left as an open problem in~\citet{JLMZ18}. 

Nevertheless, we prove that TS admits the same regret upper bound as UCB and $\epsilon$-Greedy for any adaptive manipulation, up to constant factors. These results serve as an evidence of the intrinsic robustness of stochastic bandits to strategic manipulations, regardless of which no regret learning algorithm is used.  
\begin{theorem}\label{thm:ts-ub}
For any manipulation strategy profile $S$ of strategic arms, the regret of the Thompson Sampling principal can be bounded as
\begin{eqnarray}
\E[R(T)] \leq \sum_{i\neq i^*} \max\Big\{6B_i, \frac{72\sigma^2\ln T}{\Delta_i}\Big\} + \mathcal{O}\left(\frac{\ln T}{\Delta_i}\right).
\end{eqnarray}
\end{theorem}
\vspace{-10pt}

The proof of Theorem \ref{thm:ts-ub} is quite involved as it requires us to strictly generalize  the analysis in \citet{AG2017}, which is already involved,  and further incorporate each arm's manipulation. Here we describe the key lemma (Lemma \ref{lem:ts-ub}) that leads to the above regret lower bound, and outline its proof. All formal proofs  can be found in Appendix~\ref{app:ts}.

\begin{lemma}\label{lem:ts-ub}
For any manipulation strategy profile $S$, the expected number of times that arm $i$ is pulled up to time $T$ can be bounded as follows:
\begin{eqnarray}
\E[n_{i}(T)] \leq \max\Big\{\frac{6B_i}{\Delta_i}, \frac{72\sigma^2\ln T}{\Delta_i^2}\Big\} + \mathcal{O}\left(\frac{\ln T}{\Delta_i^2}\right). 
\end{eqnarray}
\end{lemma}
\begin{proof}[Proof Sketch]
Let us start with some  useful notation. For each arm $k\in [K]$, we pick two thresholds $x_k$ and $y_k$ such that $\mu_k \leq x_k \leq y_k \leq \mu_{i^*}$.  Let $E^\mu_k(t)$ be the event $\widetilde{\mu}_k(t-1) \leq x_k$ and $\E^\theta_k(t)$ be the event $\theta_k(t)\leq y_k$. We also denote $\mathcal{F}_t$ as the history of plays until time $t$. Let $\tau_{k,s}$ be the time step at which arm $k$ is played for the $s^{\mathrm{th}}$ time and $p_{k,t}$ be the probability that $p_{k,t} = \PP(\theta_{i^*}(t)\geq y_k \big| \mathcal{F}_{t-1})$.

The key step is to carefully decompose $\E[n_{i}(T)]$, as follows:
\vspace{-10pt}
\begin{equation}\label{eq:ts-decomposition}
\begin{aligned}
\hspace{-8pt}\E[n_{i}(T)]
\leq 1 &+ E\left[\sum_{t=K+1}^T \1\big\{I_t=i,E^{\mu}_i(t),\overline{\E^\theta_i(t)}\big\}\right]\\
&+ \sum_{t=K+1}^T \PP\left(I_t=i,E^{\mu}_i(t),\E^\theta_i(t)\right)\\
& + \E\left[\sum_{t=K+1}^T \1\big\{I_t=i,\overline{E^{\mu}_i(t)}\big\}\right].
\end{aligned}
\end{equation}

The proof then proceeds by bounding each of the above terms separately. We set $x_i = \mu_i +\frac{\Delta_i}{3}, y_i = \mu_{i^*} - \frac{\Delta_i}{3}$. The first term can be bounded by $(\frac{18\ln T}{\Delta_i^2} + 1)$ using a result of~\citet{AG2017}. The second term can be bounded by $\sum_{t=K+1}^{T-1}\E\left[\frac{1}{p_{i,\tau_{i^*,s}+1}}-1\right]$ 
We then bound each summand by the following bounds (Lemma \ref{lem:ts-ub-geometric} %
 in the Appendix):
 \vspace{-10pt}
$$
\E\left[\frac{1}{p_{i,\tau_{i^*,s}+1}} - 1\right]\leq \left\{
\begin{array}{l}
e^{11/4\sigma^2} + \frac{\pi^2}{3}, \, \,   \forall s, \\
\frac{4}{T\Delta_i^2}, \text{ if } s\geq \frac{72\ln(T\Delta_i^2)\cdot\max\{1, \sigma^2\}}{\Delta_i^2}. 
\end{array}
\right.
$$
Finally, we bound the third term by $\max\Big\{\frac{6B_i}{\Delta_i}, \frac{144\sigma^2\ln T}{\Delta_i^2}\Big\} + 1$ (Lemma \ref{lem:ts-ub-3}). 
\end{proof}

\subsection{Regret Lower Bound}

It would again be natural to consider regret under a Nash equilibrium, and perhaps dominant strategy behavior. However, the equilibrium in the game induced by a \egreedy{} or TS principal is  difficult to characterize. The main challenge comes from the  additional stochasticity due to the random exploration phases in \egreedy{} and TS.
Nevertheless, we are able to prove the following matching lower bound on regret  under \lsi manipulation by using similar ideas as in the proof of Theorem \ref{thm:ucb-egreedy-lb}.  
This shows that our upper bound is indeed tight, but does not rule out the possibility of a better regret upper bound for \egreedy and TS  when arms' manipulations are restricted to a Nash equilibrium. It remains a challenging open question to characterize the  SPE under \egreedy{} and TS.  
The lower bound generalizes to bounded rewards, as shown in Theorem \ref{thm:br-lb}. 

\begin{figure*}[tb!]
	\centering 
	\begin{subfigure}{0.3\textwidth}
		\centering
		\includegraphics[scale=0.5]{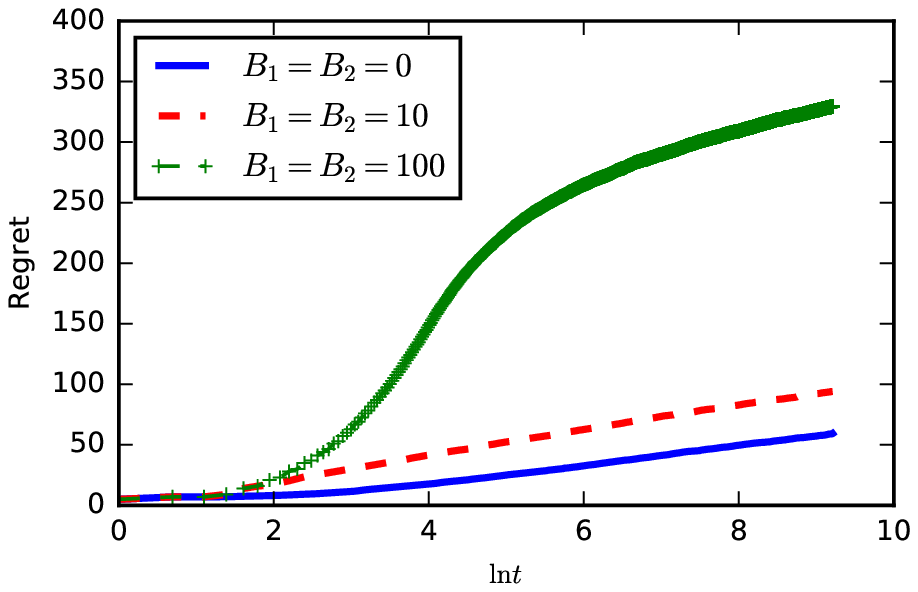}
	\end{subfigure}
	\begin{subfigure}{0.3\textwidth}
		\centering
		\includegraphics[scale=0.5]{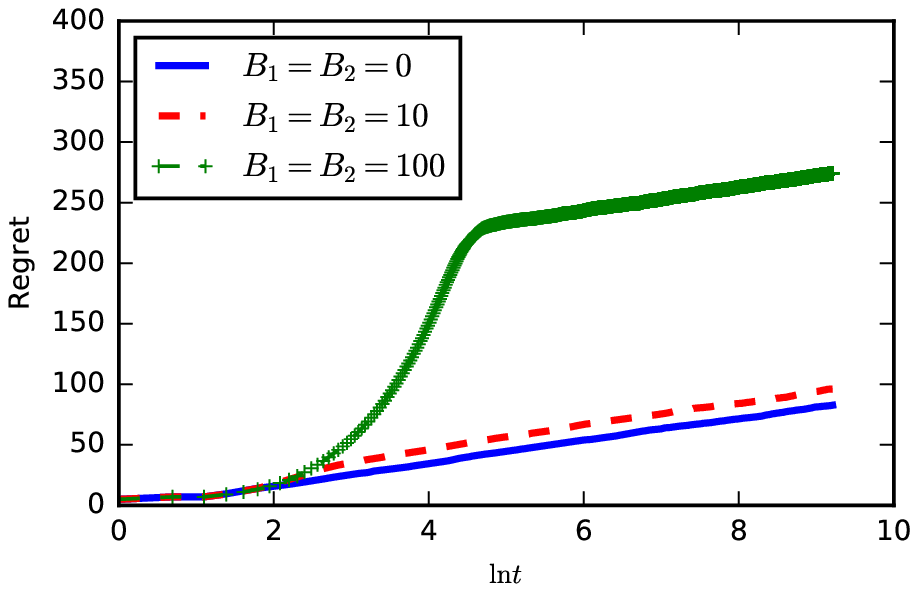}
	\end{subfigure}
	\begin{subfigure}{0.3\textwidth}
		\centering
		\includegraphics[scale=0.5]{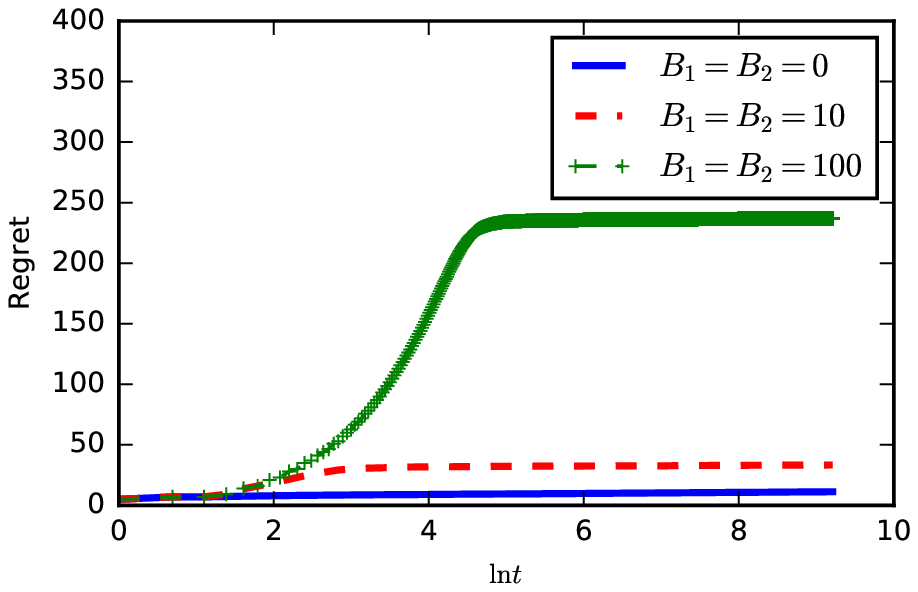}
	\end{subfigure}
	\caption{Total Regret as a function of $\ln t$ for the  UCB principal (left), $\vareps$-Greedy principal (middle), and Thompson Sampling principal (right), for three different choices for budgets of arms 1 and 2. $B_3=0$ (the strongest arm).}
	\label{fig:rgt-time}
\end{figure*}

\begin{figure*}[tb!]
	\centering
	\begin{subfigure}{0.3\textwidth}
		\centering
		\includegraphics[scale=0.5]{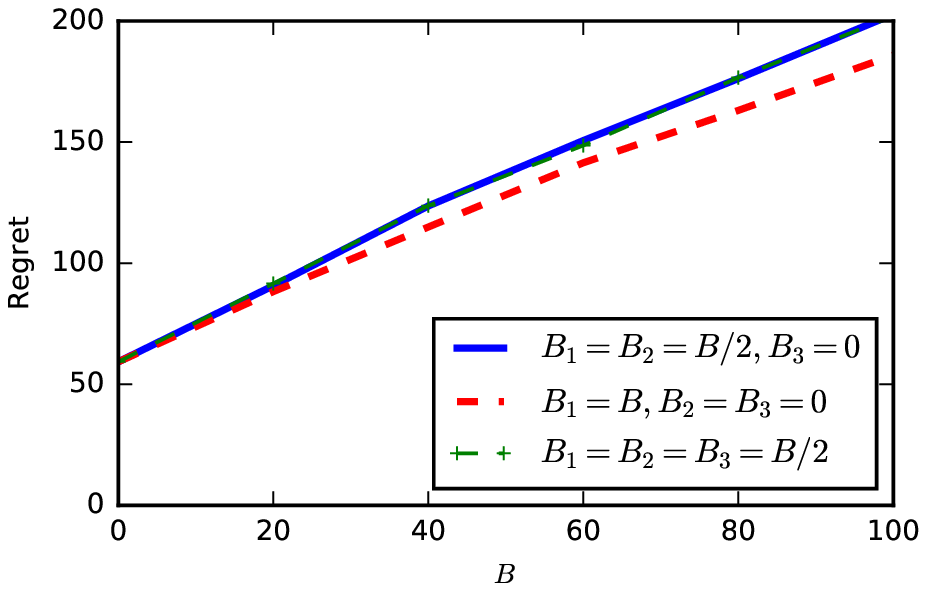}
	\end{subfigure}
	\begin{subfigure}{0.3\textwidth}
		\centering
		\includegraphics[scale=0.5]{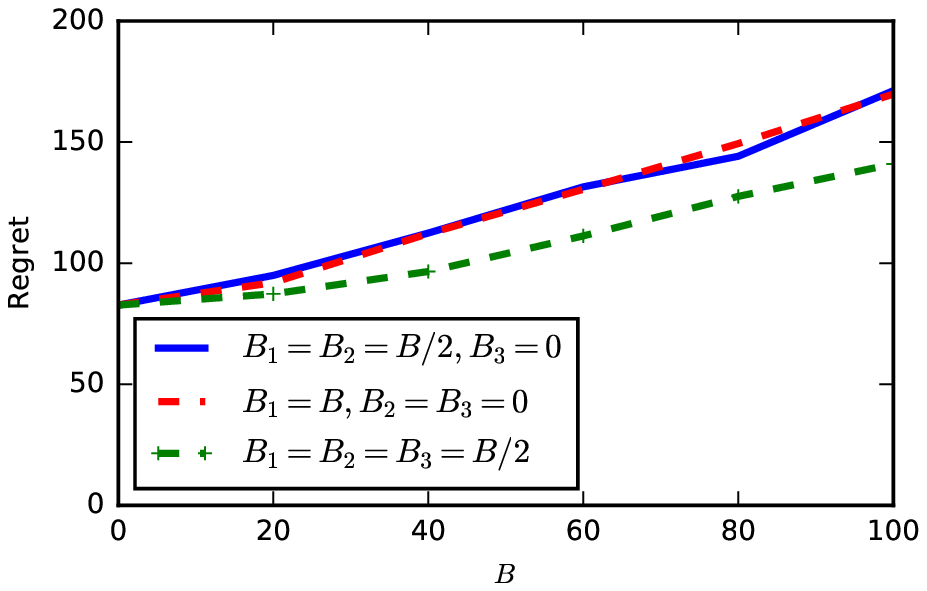}
	\end{subfigure}
	\begin{subfigure}{0.3\textwidth}
		\centering
		\includegraphics[scale=0.5]{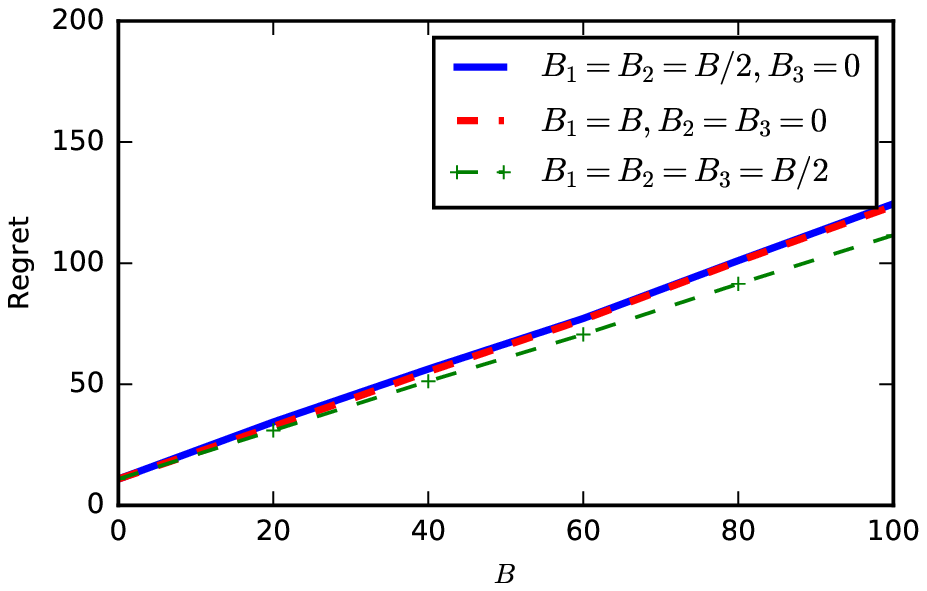}
	\end{subfigure}
	\caption{Total regret over $T=10^4$ periods as a function of total budget $B$ of arms 1 and 2,  for the UCB principal (left), $\vareps$-Greedy principal (middle), and Thompson Sampling principal (right), for three different choices of how to divide the budget, and also allowing arm 3 to have budget in one scenario. }
	\label{fig:rgt-budget}
	\vspace{-10pt}
\end{figure*} 
\begin{proposition}\label{prop:ts-lb} %
Suppose the principal runs \egreedy{}\footnote{$\eps_t = \min\{1, \frac{cK}{t}\}$ where $c = \max\{20, \frac{36\sigma^2}{\underline{\Delta}}\}$} or Thompson Sampling and each strategic arm uses $\mathtt{LSI}$. For any $\sigma$-sub-Gaussian reward distributions on arms, the regret of the principal satisfies,
$$
\E[R(T)] \geq \underline{\Delta}\sum_{i\neq i^*}\frac{B_i}{2\Delta_i} - \mathcal{O}\left(\frac{\ln T}{\underline{\Delta}}\right). 
$$
\end{proposition}

\vspace{-10pt}
\section{Simulations}\label{sec:simulation}

In this section, we provide the results of simulations to validate our theoretical results. We only present only a representative sample here,  and provide additional results in Appendix~\ref{app:simulation}.

\noindent {\bf Setup. } There are three arms, with reward distributions $\mathcal{N}(\mu_1, \sigma^2)$, $\mathcal{N}(\mu_2, \sigma^2)$ and $\mathcal{N}(\mu_3, \sigma^2)$, respectively. %
We assume that $\mu_1 < \mu_2 < \mu_3$.
In the $\vareps$-Greedy algorithm, we set $\vareps_t = \min\{1,\frac{4}{t}\}$. Throughout the simulations, we fix $\mu_1 = 5$, $\mu_2 = 8$, $\mu_3 = 10$, and $\sigma=1$. All the arms use the  $\mathtt{LSI}$ strategy. We run each bandit algorithm for $T=10^4$ rounds, and this forms one trial. We repeat for $100$ trials, and report the average results over these trials.

\noindent {\bf Regret of principal with different budgets. } We consider the regret  of UCB, $\vareps$-Greedy and Thompson Sampling with different budgets among the  arms. For each algorithm, arm 1 and arm 2 have the same budget $B_i$, chosen from $\{0, 10, 100\}$. As explained earlier, it is WLOG to assume arm 3 has zero budget. %
We show the regret as a function of $\ln t$ in Figure~\ref{fig:rgt-time}. 
We observe that for small budgets (i.e., $B_i = 0,10$), the $\Theta( \ln t) $ term dominates the regret, whereas for large budgets, the budget term $B_i$ comes to dominate the regret as $t$ becomes large. This is why we see a turning point in the regret curve
for $B_1=B_2=100$, where the regret transitions to a relatively flat curve
since the  budget is fixed.
Interestingly, we find that Thompson sampling performs better than both UCB and $\vareps$-Greedy in this strategic manipulation scenario.

\noindent {\bf Regret is linear with total budget.} We validate that the regret achieved by each stochastic bandit algorithm with strategic manipulations is linear in the total budget available to the strategic arms. We vary the budget $B=B_1+B_2$ available to arms 1 and 2, and consider three settings: (1) $B_1 = B_2 = B/2, B_3 = 0$, (2) $B_1=B, B_2=B_3=0$, and (3) $B_1 = B_2 = B_3 = B/2$. %
For setting (1), we equally split the budget to arm 1 and arm 2. For setting (2), we give all the budget to arm 1. For setting (3), we also give the optimal arm some budget (and assume arm 3 uses strategy $\mathtt{LSI}$), and want to understand the effect of the budget of the optimal arm.

Figure~\ref{fig:rgt-budget} shows the regret of each algorithm at the end of the $T=10^4$ rounds, as budget $B= B_1 + B_2$ varies. The regret is generally linearly increasing with $B$, validating the theoretical findings. Interestingly, even if the optimal arm also has available budget, the regret still increases as the budget for arms 1 and 2 increase. In fact, the regret in this case, where the optimal arm also has budget, is similar to that when it does not, and the budget on optimal arm 3 does not affect the regret much.  This is because the optimal arm will in any case be pulled many times, and its budget will be diluted significantly in later rounds, so that it  has only a small
effect on regret.

\section*{Acknowledgments}
This work is supported
in part through NSF award CCF-1841550, as well as a Google Fellowship for Zhe Feng. Haifeng Xu is supported by a Google Faculty Award. We would like to thank anonymous reviewers for their helpful feedback.

\bibliographystyle{icml2020}
\bibliography{agt}

\onecolumn
\newpage
\appendix
\begin{center}
	{
		\Large
		\textbf{
			The Intrinsic Robustness of Stochastic Bandits to Strategic Manipulation}
		~\\
		~\\	
		Appendix
	}
\end{center}

\section{Useful Definitions and Inequalities}\label{app:facts}
\begin{definition}[$\sigma$-sub-Gaussian]\label{def:sub-gaussian}
A random variable $X\in \R$ is said to be sub-Gaussian with variance proxy $\sigma^2$ if $\E\left[X\right]=\mu$ and satisfies,
\begin{align*}
\E\left[\exp(s(X-\mu))\right] \leq \exp\left(\frac{\sigma^2 s^2}{2}\right), \forall s\in \R
\end{align*}
Note the distribution defined on $[0,1]$ is a special case of $1/2$-sub-Gaussian.
\end{definition}

\begin{fact}\label{fact:sub-gaussian}
	Let $X_1, X_2,\cdots,X_n$ i.i.d drawn from a $\sigma$-sub-Gaussian, $\overline{X} = \frac{1}{n}\sum_{i=1}^n X_i$ and $\E[X]$ be the mean, then
	\begin{eqnarray*}
		\PP\left(\overline{X} - \E[X] \geq a\right) \leq e^{-na^2/2\sigma^2}~~\text{~and~}\PP\left(\overline{X} - \E[X] \leq-a\right) \leq e^{-na^2/2\sigma^2}
	\end{eqnarray*}
\end{fact}

\begin{fact}[Harmonic Sequence Bound]\label{fact:harmonic-seq}
	For $t_2> t_1 \geq 2$, we have
	\begin{align*}
	\ln\frac{t_2}{t_1}\leq \sum_{t=t_1}^{t_2}\frac{1}{t} \leq \ln\left(\frac{t_2}{t_1-1}\right)
	\end{align*}
\end{fact}

\begin{fact}\label{fact:gaussian-concentration}
For a Gaussian distributed random variable $Z$ with mean $\mu$ and variance $\sigma^2$, for any $z$,
\begin{align*}
\PP\left(|Z-\mu|>z\sigma\right)\leq \frac{1}{2}e^{-z^2/2}
\end{align*}
\end{fact}

\begin{lemma}[Theorem 3 in \cite{ACF02}]\label{lem:e-greedy-concentration}
	In $\varepsilon$-Greedy, for any arm $k\in [K], t > K, n\in \mathbb{N}_+$, we have
	\begin{align*}
	\PP\left(\widehat{\mu}_k(t-1) \leq \mu_{k} - \frac{\Delta_k}{n} \right) \leq x_t \cdot e^{-x_t/5} + \frac{2\sigma^2n^2}{\Delta_k^2} e^{-\Delta_k^2 \lfloor x_t\rfloor/2\sigma^2n^2}, ~~~\text{and}\\
	\PP\left(\widehat{\mu}_{i^*}(t-1) \geq \mu_{i^*} + \frac{\Delta_k}{n} \right) \leq x_t \cdot e^{-x_t/5} + \frac{2\sigma^2n^2}{\Delta_k^2} e^{-\Delta_k^2 \lfloor x_t\rfloor/2\sigma^2n^2},
	\end{align*}
	where $x_t = \frac{1}{2K} \sum_{s=K+1}^t\varepsilon_s$.
\end{lemma}

\section{Ommited Proofs in Section~\ref{sec:ucb-greedy}}\label{app:ucb-greedy}

\subsection{Proof of Lemma~\ref{lem:ucb-ub}}\label{app:ucb}
\begin{proof}
	Let $C_i(T) =  \max\left\{\frac{81\sigma^2\ln T}{\Delta_i^2}, \frac{3B_i}{\Delta_i} \right\}$. By Fact~\ref{fact:sub-gaussian}, we have for any $s \geq 1$ and $\ell \geq C_i(T)$
	\begin{equation}\label{eq:ucb-concentration}
	\begin{aligned}
	\forall k,~~\PP\left(\mu_k-\widehat{\mu}_k(t-1) \geq 3\sigma\sqrt{\frac{\ln \new{T}}{n_k(t-1)}}\Big|n_k(t-1)=s\right) &\leq \frac{1}{\new{T}^{9/2}}\\
	\PP\left(\widehat{\mu}_i(t-1) - \mu_i \geq \frac{\Delta_i}{3}\Big|n_i(t-1)=\ell\right)& \leq \frac{1}{T^{9/2}}
	\end{aligned}
	\end{equation}
	
We first decompose $\E[n_i(T)]$ as follows,
	\begin{equation}\label{eq:ucb-ub-1}
	\begin{aligned}
	\E\left[n_i(T)\right]
	&\leq 1 + \E\left[\sum_{t=K+1}^T \1\{I_t=i, n_i(t-1)\leq C_i(T)\}\right] + \E\left[\sum_{t=K+1}^T \1\{I_t=i, n_i(t-1)\geq C_i(T)\}\right]\\
	&\leq 1 + C_i(T) + \E\left[\sum_{t=K+1}^T \1\{I_t=i, n_i(t-1)\geq C_i(T)\}\right]\\
	&\leq 1 + C_i(T) + 
	\sum_{t=K+1}^T \PP\left(\mathrm{UCB}_i(t) + \frac{\beta^{(i)}_{t-1}}{n_i(t-1)} \geq\mathrm{UCB}_{i^*}(t), n_i(t-1)\geq C_i(T)\right)\\
	\end{aligned}
	\end{equation}
We then bound the probability $\PP\left(\mathrm{UCB}_i(t) + \frac{\beta^{(i)}_{t-1}}{n_i(t-1)} \geq\mathrm{UCB}_{i^*}(t), n_i(t-1)\geq C_i(T)\right)$ by union bound, and decompose this probability term as follows,
	\begin{equation}\label{eq:ucb-ub-2}
	\begin{aligned}
	& \PP\left(\mathrm{UCB}_i(t) + \frac{\beta^{(i)}_{t-1}}{n_i(t-1)} \geq\mathrm{UCB}_{i^*}(t), n_i(t-1)\geq C_i(T)\right)\\
	& \leq \sum_{s=1}^{t-1} \sum_{\ell\geq C_i(T)}^{t-1} \PP\left(\mathrm{UCB}_i(t) + \frac{\beta^{(i)}_{t-1}}{n_i(t-1)} \geq\mathrm{UCB}_{i^*}(t)\Big|n_i(t-1)=\ell, n_{i^*}(t-1)=s\right).
	\end{aligned}
	\end{equation}
	What remains is to upper bound the summand in the above term. Consider for $1\leq s\leq t-1$ and $C_i(T)\leq\ell \leq t-1$, we have
	\begin{small}
	\begin{align*}
	&\PP\left(\mathrm{UCB}_i(t) + \frac{\beta^{(i)}_{t-1}}{n_i(t-1)} \geq\mathrm{UCB}_{i^*}(t)\Big|n_i(t-1)=\ell, n_{i^*}(t-1)=s\right)\\
	&~~~\leq \PP\left(\widehat{\mu}_i(t-1) + 3\sigma\sqrt{\frac{\ln \new{T}}{n_i(t-1)}} + \frac{\Delta_i}{3}\geq\widehat{\mu}_{i^*}(t-1)+3\sigma\sqrt{\frac{\ln \new{T}}{n_{i^*}(t-1)}}\Big| n_i(t-1)=\ell, n_{i^*}(t-1)=s\right)\\
	&~~~\leq \PP\left(\widehat{\mu}_i(t-1) + \frac{\Delta_i}{3} + \frac{\Delta_i}{3}\geq \widehat{\mu}_{i^*}(t-1)+3\sigma\sqrt{\frac{\ln \new{T}}{n_{i^*}(t-1)}}\Bigg|n_i(t-1)=\ell, n_{i^*}(t-1)=s \right)\\
	\end{align*}
	\end{small}
	The first inequality relies on the fact that $\ell\geq C_i(T)\geq \frac{3B_i}{\Delta_i}\geq \frac{3\beta^{(i)}_{t-1}}{\Delta_i}$ and second inequality holds because $\ell\geq C_i(T)\geq \frac{81\sigma^2\ln T}{\Delta_i^2}$.
	By union bound and Equation~(\ref{eq:ucb-concentration}), we can further upper bound the last term in the above inequality by 
	\begin{align*}&\PP\left(\widehat{\mu}_i(t-1) - \mu_i \geq \frac{\Delta_i}{3}\Big| n_i(t-1)=\ell\right) + \PP\left(\mu_{i^*}-\widehat{\mu}_{i^*}(t-1) \geq 3\sigma\sqrt{\frac{\ln \new{T}}{n_{i^*}(t-1)}}\Big|n_{i^*}(t-1)=s\right)\\
	&\leq \frac{1}{T^{9/2}}+\frac{1}{\new{T}^{9/2}} \leq \frac{2}{T^{9/2}}\end{align*}
	Combining Equations~(\ref{eq:ucb-ub-1}) and the fact that $$\sum_{t=K+1}^T\sum_{s=1}^{t-1}\sum_{\ell\geq C_i(T)}^{t-1} \frac{2}{t^{T/2}}\leq \sum_{t=K+1}^T \frac{2}{T^{2}}\leq 2,$$ we complete the proof.
\end{proof}

\subsection{Proof of Theorem~\ref{thm:opt-attack}}\label{app:lsi-optimality}

We begin with a few notations. %
Let $I_t^{S}$ denote the arm  being pulled at time $t$ for any  investment strategy $S$, and $Z^{S}_t = \{I_1^{S}, \cdots, I_t^{S} \}$ denote the \emph{sequence} of arms being pulled \emph{up to} time $t$.  %
Note that $Z^{S}_t = \{ I_1^{S}, \cdots, I_t^{S} \}$ can be viewed as a stochastic process for any $t$. \new{Let $S^{(-i)}$ denote the investment strategies of all arms excluding arm $i$. In addition, we denote by $(\mathtt{LSI}, S^{(-i)})$ the strategy that arm $i$ uses \lsi strategy and the other arms adopt $S^{(-i)}$. For each arm $j\neq i$, $S^{(j)}$ only depends on its own history, which means given fixed strategies $S^{(-i)}$, at any time $t$, each of the arms $j\neq i$ will invest the same budget if it has been pulled the same times and the true rewards are the same up to time $t$.}

Our proof of Theorem \ref{thm:opt-attack} relies on a carefully chosen coupling of the two stochastic processes  $Z_T^{S_1} , Z_T^{S_2}  $ induced by different investment strategies $S_1, S_2$, respectively. 

\begin{definition}[{\bf Arm Coupling}]\label{def:exact-couple}
Given any two investment strategies $ S_1, S_2$, the {\bf Arm Coupling} of   $Z_T^{S_1} $ and $Z_T^{S_2} $ is a coupling of these two stochastic processes such that  the reward of any arm $k \in [K]$ when pulled for the same times is the same in these two random processes. In this case, we also say $Z_T^{S_1} = \{ I_1^{S_1}, \cdots, I_T^{S_1} \}$ and $Z_T^{S_2} = \{ I_1^{S_2}, \cdots, I_T^{S_2} \}$   are \bf{Arm-Coupled}. 
\end{definition}

Our goal is to compare $\left(\mathtt{LSI}, S^{(-i)}\right)$ and any other strategy $S = (S^{(i)}, S^{(-i)})$ for arm $i$, using {\bf Arm Coupling}. \new{In the remainder of this proof we will always fix all other arms' manipulation strategy $S^{(-i)}$. Thus for convenience we simply omit $S^{(-i)}$ in the superscript and use  $Z^{\mathtt{LSI}}_{t}$ and   $Z^{S^{(i)}}_{t}$ to denote the two stochastic sequences of our interests. Let $Z^{\mathtt{LSI}}_{t:t'}$ denote the stochastic process from time $t$ to time $t'$ under $\left(\mathtt{LSI}, S^{(-i)}\right)$ manipulation, and similarly for $Z^{S^{(i)}}_{t:t'}$. Similar notations and simplifications are used  for $n_i$.  We first show \lsi is the dominant strategy for the arm when principal runs UCB algorithm, given any history $h_{t-1}$. Hence $\mathtt{LSI}$ is a dominant-strategy SPE.}

The following lemma shows an interesting property about the two arm sequences $Z^{\mathtt{LSI}}_t $ and  any $Z^{S^{(i)}}_t$ pulled under these two different investment strategies. That is, under {\bf Arm Coupling}, all the arms ---  except for the special arm $i$ --- will be pulled according to the same order after time $t$, given any history $h_{t-1}$. 

\begin{lemma}\label{lem:LSI-optimal-ucb-1}
Suppose $t \geq K$ and the principal runs UCB algorithm.  Let  $Z^{\mathtt{LSI}}_{t:t'} (-i)$ [resp. $Z^{S^{(i)}}_{t:t'}(-i)$] denote the subsequence of $Z^{\mathtt{LSI}}_{t:t'}$ [resp. $Z^{S^{(i)}}_{t:t'}$]  after deleting all $i$'s in the sequence. Then given any history $h_{t-1}$ and time $t$, under {\bf Arm Coupling}, either  $Z^{\mathtt{LSI}}_{t:t'}(-i)$ is a subsequence of $Z^{S^{(i)}}_{t:t'}(-i)$ or vice versa. 
	
\end{lemma}
\begin{proof}
We prove by induction on $t'$. When $t' = t$, if $I_t^{(\mathtt{LSI}, S^{(-i)})}$ or $I_t^{S}$ is $i$, the conclusion holds  trivially. If $I_t^{S} = k \not = i$, then $k$ is the largest UCB term. Since the history $h_{t-1}$ is fixed, UCB terms of each arm must be the same, thus, if $I_t^{S}=k$, then $I_t^{\mathtt{LSI}}=k$, as desired. 

Now, assume the lemma holds for some $t' (> t)$, and we now consider the case $t'+1$. This follows a case analysis.

If $n_{i}^{\mathtt{LSI}} (t:t')= n_{i}^{S^{(i)}}(t:t')$, then we know that $Z^{\mathtt{LSI}}_{t:t'} (-i)$ and  $Z^{S^{(i)} }_{t:t'}(-i)$ have the same length. Since one of them is a subsequence of the other by induction hypothesis, this implies that they are the same sequence. If one of  $I_{t'+1}^{\mathtt{LSI}}, I_{t'+1}^{S^{(i)}}$ equals $i$, say, e.g., $I_{t'+1}^{\mathtt{LSI}} = i$,  then  $Z^{\mathtt{LSI}}_{t:t'+1} (-i) =  Z^{\mathtt{LSI}}_{t:t'} (-i) = Z^{S^{(i)}}_{t:t'} (-i)$ which is a subsequence of $ Z^{S^{(i)}}_{t:t'+1}(-i)$,  as desired. If both $I_{t'+1}^{\mathtt{LSI}}, I_{t'+1}^{S^{(i)}}$ are not equal to $i$, then we claim that they must be the same arm. This is  because they are the arm with the highest UCB index after round $t$. Since $Z^{\mathtt{LSI}}_{t:t'}(-i)$ and  $Z^{S^{(i)}}_{t:t'}(-i)$ are the same sequence of arms, each arms are pulled by exactly the same time in both stochastic processes from $0$ to $t'$, given the fixed history $h_{t-1}$. Moreover, due to {\bf Arm Coupling}, their rewards are also the same. \new{Given the fixed strategies of the other arms $S^{(-i)}$, } their manipulations will also be the same. Therefore, the arm with the highest \emph{modified UCB terms} must also be the same.  Therefore, we have  $Z^{\mathtt{LSI}}_{t:t'+1} (-i) =  Z^{S^{(i)}}_{t: t'+1}(-i)$, as desired.
	
If $n_{i}^{\mathtt{LSI}}(t:t') > n_{i}^{S^{(i)}}(t:t')$, then we know that $Z^{\mathtt{LSI}}_{t:t'} (-i)$ is a  strict subsequence of $Z^{S^{(i)}}_{t:t'}(-i)$. Let $l = |Z^{\mathtt{LSI}}_{t:t'} (-i)|$ denote the length of $Z^{\mathtt{LSI}}_{t:t'} (-i)$, and $\tilde{k}$ denote the $(l+1)$th element in $Z^{S^{(i)} }_{t:t'}(-i)$. We claim that $I_{t'+1}^{\mathtt{LSI}}$ must be either $i$ or $\tilde{k}$, which implies $Z^{\mathtt{LSI}}_{t:t'+1} (-i)$ is a  subsequence of $Z^{S^{(i)} }_{t:t'+1}(-i)$ as desired.  In particular, if $I_{t'+1}^{\mathtt{LSI}} \not = i$, then the fact that $\tilde{k}$ is the $(l+1)$th element in $Z^{S^{(i)}}_{t:t'}(-i)$ implies that $\tilde{k}$ has the highest \emph{modified UCB term} among all arms in $[K] \setminus  \{ i \}$ when these arms are pulled according to sequence $Z_{t}^{S^{(i)}}$. Following the same argument above and {\bf Arm Coupling}, we know that  $I_{t'+1}^{\mathtt{LSI}}$, the arm with the highest \emph{modified UCB term}, must equal $\tilde{k}$ if it does not equal $i$. 

The case of $n_{i}^{\mathtt{LSI}}(t:t') < n_{i}^{S^{(i)}}(t:t')$ can be argued similarly. This concludes the proof of the lemma.
\end{proof}

The following lemma shows that under {\bf Arm Coupling},  the number of times that arm $i$ is pulled up to time $T$ under strategy $\mathtt{LSI}$ is always at least that under any other investment strategy $S^{(i)}$. %
\begin{lemma}\label{lem:LSI-optimal-ucb-2}
When the principal runs UCB algorithm, under {\bf Arm Coupling}, given any history $h_{t-1}$ and time $t$, we have  $n_{i}^{\mathtt{LSI}}(t:T) \geq n_{i}^{S^{(i)}}(t:T)$ with probability $1$ for any investment strategy $S$ and $T \geq t \geq K$. 
\end{lemma}
\begin{proof}
We still prove through induction. Given any fixed $t\geq K$ and history $h_{t-1}$, for $T=t$, it holds trivially since if $I_t^{S^{(i)}} = i$ then $I_t^{\mathtt{LSI}}$ must be $i$.  We assume this lemma is true for $t' = T -1> t$. For $t' = T$, we consider the following two cases.
\begin{enumerate}
\item If $n_{i}^{\mathtt{LSI}}(t: T-1) > n_{i}^{S^{(i)}}(t:T-1)$, then $n_{i}^{\mathtt{LSI}}(t:T) \geq n_{i}^{\mathtt{LSI}}(t:T-1) \geq n_{i}^{S^{(i)}}(t:T-1) + 1\geq n_{i}^{S^{(i)}}(t:T)$, as desired. 
\item If $n_{i}^{\mathtt{LSI}}(t:T-1) = n_{i}^{S^{(i)}}(t:T-1)$, then Lemma~\ref{lem:LSI-optimal-ucb-1} implies that $Z_{t:T-1}^{\mathtt{LSI}}$ and $Z^{S^{(i)}}_{t:T-1}$ are the same sequence. Therefore, the $\mathrm{UCB}$ term for each arm $k\in [K]$ (excluding arm $i$) for $\mathtt{LSI}$ and $S^{(i)}$ are the same at time $T$. For arm $i$, we have 
\begin{align*}
\widehat{\mathrm{UCB}}_i^{(\mathtt{LSI}, S^{(-i)})}(T) &= \mathrm{UCB}_i^{(\mathtt{LSI}, S^{(-i)})}(T) + \frac{B_i}{n_{i}^{(\mathtt{LSI}, S^{(-i)})}(T-1)} = \mathrm{UCB}_i^{S}(T) + \frac{B_i}{n_{i}^{S}(T-1)} \\
&\geq \mathrm{UCB}_i^{S}(T) + \frac{\beta^{(i)}_{T-1}}{n_{i}^{S}(T-1)} = \widehat{\mathrm{UCB}}_i^{S}(T),
\end{align*}
This  implies that if $I_T^{S}=i$, then we must also have $I_{T}^{(\mathtt{LSI}, S^{(-i)})} = i$. Then $n_{i}^{(\mathtt{LSI}, S^{(-i)})}(T)\geq n_{i}^{S}(T)$ still holds. 
\end{enumerate}
\noindent To sum up,  $n_{i}^{(\mathtt{LSI}, S^{(-i)})}(t:T)\geq n_{i}^{S}(t:T)$ holds with probability $1$, concluding the proof. 
\end{proof}

\subsection{Proof of Theorem~\ref{thm:ucb-egreedy-lb}}\label{app:ucb-egreedy-lb}

We show the lower bound of the regret by deriving the upper bound of the expected number of times that arm $i^*$ being pulled, which is summarized in Lemma~\ref{lem:LSI-opt-arm-ub}. Given Lemma~\ref{lem:LSI-opt-arm-ub} and Eq.~\eqref{eq:regret-lb}, it is straightforward to conclude Theorem~\ref{thm:ucb-egreedy-lb} for UCB principal.
\begin{lemma}\label{lem:LSI-opt-arm-ub}
Suppose each strategic arm $i (i\neq i^*)$ uses $\mathtt{LSI}$ and $\underline{\Delta} = \min_{i\neq i^*}\Delta_i$, the expected number of times that optimal arm $i^*$ being pulled up to time $T$ is bounded by,
\begin{align*}
\E\left[n_{i^*}(T)\right] \leq T - \sum_{i\neq i^*}\frac{B_i}{2\Delta_i} + \mathcal{O}\left(\frac{\ln T}{\underline{\Delta}^2}\right)
\end{align*}
\end{lemma}

\begin{proof}
Let $\underline{\Delta}=\min_{i\neq i^*}\Delta_i, C(T) = \frac{36\sigma^2\ln T}{\underline{\Delta}^2}, D_i =\frac{B_i}{2\Delta_i}$. First, by Fact~\ref{fact:sub-gaussian}, we have for any $\ell\geq C(T)$, $s\geq 1$ and any $i$,
\begin{equation}\label{eq:ucb-aux-ie}
\begin{aligned}
\PP\left(\mu_i - \widehat{\mu}_{i}(t-1) \geq 3\sigma\sqrt{\frac{\ln T}{n_i(t-1)}}\Big| n_i(t-1)=s\right) \leq \frac{1}{T^{9/2}}\\
\PP\left(\widehat{\mu}_{i^*}(t-1) - \mu_{i^*} \geq \frac{\Delta_i}{2}\Big|n_{i^*}(t-1)=\ell\right) \leq \exp\left(-\frac{\ell\Delta_i^2}{8\sigma^2}\right) \leq \exp\left(-\frac{C(T)\Delta_i^2}{8\sigma^2}\right) \leq \frac{1}{T^{9/2}}
\end{aligned}
\end{equation}
First, we decompose $\E\left[n_{i^*}(T)\right]$ as follows,
\begin{equation}\label{eq:ucb-ub-opt-arm}
\begin{aligned}
\E\left[n_{i^*}(T)\right]
&\leq 1 + \E\left[\sum_{t=K+1}^T \1\{I_t=i^*, n_{i^*}(t-1)\leq C(T)\} \right] +\E\left[\sum_{t=K+1}^T \1\{I_t=i^*, n_{i^*}(t-1)\geq C(T)\} \right]\\
&\leq  1 + \E\left[\sum_{t=K+1}^T \1\{I_t=i^*, n_{i^*}(t-1)\leq C(T)\}\right] \\
&\hspace{1cm} +  \E\left[\sum_{t=K+1}^T \1\{I_t=i^*, n_{i^*}(t-1)\geq C(T), \forall i\neq i^*, n_i(t-1)\geq D_i\}\right]\\
&\hspace{1cm} + \E\left[\sum_{t=K+1}^T \1\left(I_t=i^*, n_{i^*}(t-1)\geq C(T), \exists i\neq i^*,  n_i(t-1)\leq D_i\right)\right]\\
\end{aligned}
\end{equation}
For the first term in the above decomposition, it can be trivially bounded by $C(T)$. For the second term, since $n_{i^*}(t)\leq T - \sum_{i\neq i^*} n_{i}(t), \forall t$, we have
\begin{align*}
&\E\left[\sum_{t=K+1}^T \1\{I_t=i^*, n_{i^*}(t-1)\geq C(T), \forall i\neq i^*, n_i(t-1)\geq D_i\}\right] \\
&\hspace{1cm}\leq \E\left[\sum_{t=K+1}^T \1\{I_t=i^*, n_{i^*}(t-1)\leq T-\sum_{i\neq i^*}D_i\}\right] \leq T - \sum_{i\neq i^*} \frac{B_i}{2\Delta_i}
\end{align*}
What remains is to bound the third term in Equations~\eqref{eq:ucb-ub-opt-arm}. By union bound, we have
\begin{align*}
&\E\left[\sum_{t=K+1}^T \1\left(I_t=i^*, n_{i^*}(t-1)\geq C(T), \exists i\neq i^*,  n_i(t-1)\leq D_i\right)\right]\\
& \hspace{1cm}=\sum_{i\neq i^*} \sum_{t=K+1}^T \PP\left(I_t=i^*, n_{i^*}(t-1)\geq C(T), n_i(t-1)\leq D_i\right)
\end{align*}

Note $I_t = i^*$ implies $\mathrm{UCB}_{i^*}(t)\geq \widehat{\mathrm{UCB}}_i(t)$, combining the facts that $3\sigma\sqrt{\frac{\ln T}{n_{i^*}(t-1)}}\leq \underline{\Delta}/2$ and $\frac{B_i}{n_i(t-1)}\geq 2\Delta_i$ and standard union bound, we have
\begin{small}
\begin{equation}
\begin{aligned}
&\PP\left(I_t=i^*, n_{i^*}(t-1)\geq C(T), n_i(t-1)\leq D_i\right)\\
&\leq \sum_{s=1}^{D_i\wedge t-1}\sum_{\ell\geq C(T)}^{t-1}\PP\left(\widehat{\mu}_{i^*}(t-1) + 3\sigma\sqrt{\frac{\ln T}{n_{i^*}(t-1)}}\geq\mathrm{UCB}_i(t)+ \frac{B_i}{n_i(t-1)}\Big|n_{i^*}(t-1)=\ell, n_{i}(t-1)=s\right)\\
& \leq \sum_{s=1}^{D_i\wedge t-1}\sum_{\ell\geq C(T)}^{t-1}\PP\left(\widehat{\mu}_{i^*}(t-1) + \frac{\Delta_i}{2}\geq\widehat{\mu}_{i}(t-1) + 3\sigma\sqrt{\frac{\ln T}{n_i(t-1)}} + 2\Delta_i,\Big| n_{i^*}(t-1)=\ell, n_i(t-1)=s\right)\\
& \leq \sum_{s=1}^{D_i\wedge t-1}\sum_{\ell\geq C(T)}^{t-1}\PP\left(\widehat{\mu}_{i^*}(t-1) -\mu_{i^*} \geq \frac{\Delta_i}{2}\Big| n_{i^*}(t-1)=\ell\right) + \PP\left(\mu_{i} - \widehat{\mu}_{i}(t-1)\geq 3\sigma\sqrt{\frac{\ln T}{n_i(t-1)}}\Big|n_i(t-1)=s\right) 
\end{aligned}
\end{equation}
\end{small}
The last inequality is based on union bound, if both $\widehat{\mu}_{i^*}(t-1) -\mu_{i^*} < \underline{\Delta} / 2$ and $\mu_{i} - \widehat{\mu}_{i}(t-1) < 3\sigma\sqrt{\frac{\ln T}{n_i(t-1)}}$ hold when $n_{i^*}(t-1)=\ell, n_i(t-1)=s$, then
\begin{align*}
\widehat{\mu}_{i^*}(t-1) + \frac{\Delta_i}{2} & < \mu_{i^*} + \frac{\underline{\Delta}}{2} + \frac{\Delta_i}{2} \leq \mu_i + \Delta_i + \Delta_i\\
& < \widehat{\mu}_{i}(t-1) +  3\sigma\sqrt{\frac{\ln T}{n_i(t-1)}} +2\Delta_i
\end{align*}  
Given Equation~(\ref{eq:ucb-aux-ie}), we have $$\PP\left(I_t=i^*, n_{i^*}(t-1)\geq C(T), n_i(t-1)\leq D_i\right)\leq \sum_{s=1}^{t-1}\sum_{\ell=1}^{t-1} \frac{2}{T^{9/2}}\leq \frac{2}{T^2}$$
Combining Equation~(\ref{eq:ucb-ub-opt-arm}), we get
\begin{align*}
&\E\left[n_{i^*}(T)\right] \leq 1 + C(T) + T - \sum_{i\neq i^*}\frac{B_i}{2\Delta_i}+\sum_{i\neq i^*}\sum_{t=K+1}^T \frac{2}{T^{2}}\\
&\hspace{1cm} = T + \frac{36\sigma^2\ln T}{\underline{\Delta}^2}- \sum_{i\neq i^*}\frac{B_i}{2\Delta_i} + 1 + \frac{2(K-1)}{T}
\end{align*}
\end{proof}

Combining Lemma~\ref{lem:LSI-opt-arm-ub} and Eq/~\ref{eq:regret-lb}, we complete the proof for Theorem~\ref{thm:ucb-egreedy-lb}.

\subsection{Proof of Theorem~\ref{thm:br-lb}}
To prove Theorem~\ref{thm:br-lb}, we first show the following Lemma.
\begin{lemma}\label{lem:br-ub-opt-arm}
	Suppose all the strategic arms use $\mathtt{LSIBR}$, and let time step $n$ be the last time that a strategic arm spend budget for some $n \leq T$. Then for the three algorithms we consider (UCB, \egreedy \, and TS), the expected number of plays of the optimal arm $i^*$ from time $n+1$ to $T$  is bounded by,
	\begin{align*}
	\E\left[\sum_{t=n+1}^T \1\big\{I_t = i^*\big\}\right] \leq \E\left[n_{i^*}^{\mathtt{LSI}}(T)\right] = T - \sum_{i\neq i^*}\frac{B_i}{2\Delta_i} + \mathcal{O}\left(\frac{\ln T}{\underline{\Delta}^2}\right). 
	\end{align*}
\end{lemma}
\begin{proof}
	The proof follows a simple reduction to the setting with arms using $\mathtt{LSI}$. By using $\mathtt{LSIBR}$, any strategic arm $i$ has no budget to manipulate after (includes) time step $n+1$, which is analogous to the case that arm $i$ has no budget to manipulate after time $K+1$ using $\mathtt{LSI}$ in unbounded reward setting. Then after time $n+1$, the $\widetilde{\mu}_i(t-1) = \widehat{\mu}_i(t-1) + \frac{B_i}{n_i(t-1)}, \forall \in [K]$, which shares the same formula with it in $\mathtt{LSI}$ setting. Finally, we notice that the proofs of the upper bounds of $\E\left[\sum_{t=K+1}^T \1\{I_t = i^*\}\right]$ in $\mathtt{LSI}$ settings (Lemma~\ref{lem:LSI-opt-arm-ub},~\ref{lem:e-greedy-lb} and Theorem~\ref{thm:ts-lb-opt-arm-ub}) don't depend on the starting time step in the summand. Therefore, the proofs in these previous results can be directly applied here.
\end{proof}

Next, we prove Theorem~\ref{thm:br-lb} using the above Lemma.
\begin{proof}[\textbf{Proof of Theorem~\ref{thm:br-lb}}]
	Let $n$ be the last time step that any arm can spend the budget.
	First we show the upper bound of $\E\left[n_{i^*}^{\mathtt{LSIBR}}(T)\right]$. Note, from time $1$ to $n-1$, any strategic arm $i$ always promote its reward to $1$, which makes arm $i$ the "optimal arm" from time 1 to $n$ (the arm selection at time $n$ only depends on previous feedback). Then following the standard analysis in stochastic MAB alogrithms (UCB, $\vareps$-Greedy and Thompson Sampling), $\E\left[n_{i^*}^{\mathtt{LSIBR}}(n)\right] \leq O\Big(\frac{\ln n}{(1-\mu_{i^*})^2}\Big)$. Thus, $\E\left[n_{i^*}^{\mathtt{LSIBR}}(T)\right]$ can be bounded by,
	\begin{eqnarray*}
		\E\left[n_{i^*}^{\mathtt{LSIBR}}(T)\right] \leq \E\left[n_{i^*}^{\mathtt{LSI}}(T)\right] + \mathcal{O}\left(\frac{\ln n}{(1-\mu_{i^*})^2}\right).
	\end{eqnarray*}
	Consequently, we can show the lower bound of regret when all strategic arms use $\mathtt{LSIBR}$, as follows
	\begin{align*}
	\E\left[R^{\mathtt{LSIBR}}(T)\right] \geq\E\left[R^{\mathtt{LSI}}(T)\right] - \mathcal{O}\left(\frac{\underline{\Delta}\ln T}{(1-\mu_{i^*})^2}\right). 
	\end{align*}
\end{proof}

\section{Omitted Proofs in Section~\ref{sec:ts}}\label{app:ts}

\subsection{Proof of Theorem~\ref{thm:e-greedy-ub}}\label{app:e-greedy}
To prove this theorem, we instead prove the following Lemma~\ref{lem:e-greedy-ub} to bound $\E[n_i(T)]$ for each arm $i\neq i^*$. Given this Lemma, it is then easy to show Theorem~\ref{thm:e-greedy-ub}.
\begin{lemma}\label{lem:e-greedy-ub}
	Suppose the principal runs the $\vareps$-Greedy algorithm with $\varepsilon_t = \min\{1, \frac{cK}{T}\}$ when $t > K$, where the constant $c=\max\big\{20, \frac{36\sigma^2}{\Delta_i^2}\big\}$. Then	for any strategic manipulation strategy $S$, the expected number of times of arm $i$ being pulled up to time $T$ can be bounded by
	\begin{align*}
	\E\left[n_i(T)\right] \leq \frac{3B_i}{\Delta_i}  + \mathcal{O}\left(\frac{\ln T}{\Delta_i^2}\right).
	\end{align*}
\end{lemma}

\begin{proof}
	Let $C_i = \frac{3B_i}{\Delta_i}, x_t = \frac{1}{2K}\sum_{s=K+1}^t\epsilon_s$ and for $t\geq \lfloor cK\rfloor +1$, Given Fact~\ref{fact:harmonic-seq}, we have
	\begin{eqnarray}\label{eq:e-greedy-x_t}x_t \geq \sum_{s=K+1}^{\lfloor cK\rfloor}\frac{\epsilon_s}{2K}+ \sum_{t=\lfloor cK\rfloor+1}^t \frac{\epsilon_s}{2K} \geq \lfloor cK\rfloor - K + \frac{c}{2}\sum_{s=\lfloor cK\rfloor+1}^t \frac{1}{s} \geq \lfloor cK\rfloor - K +  \frac{c}{2}\ln\frac{t}{\lfloor cK\rfloor+1}\end{eqnarray}
	We do the decomposition for $\E[n_i(T)]$ as follows,
	\begin{equation}\label{eq:e-greedy-ub-1}
	\begin{aligned}
	\E\left[n_i(T)\right] &\leq 1 + \E\left[\sum_{t=K+1}^T \1\{I_t = i, n_i(t-1)\leq C_i\}\right] + \E\left[\sum_{t=K+1}^T \1\{I_t = i, n_i(t-1)\geq C_i\}\right]\\
	&\leq 1 + C_i + \sum_{t=K+1}^{T}\frac{\epsilon_t}{K} + \E\left[\sum_{t=K+1}^T (1-\epsilon_t)\cdot \1\Big\{\widetilde{\mu}_i(t-1)\geq \widehat{\mu}_{i^*,t-1}, n_i(t-1)\geq C_i\Big\}\right]\\
	& \leq 1 + C_i + \sum_{t=K+1}^{T}\frac{\epsilon_t}{K} + \sum_{t=\lfloor cK\rfloor+1}^T \PP\Big(\widehat{\mu}_i(t-1) + \frac{\beta_{t-1}}{n_i(t-1)}\geq \widehat{\mu}_{i^*}(t-1), n_i(t-1)\geq C_i\Big)
	\end{aligned}
	\end{equation}
	The last inequality holds because $\eps_t = 1$ when $t\leq \lfloor cK\rfloor$ and $1-\eps_t\leq 1, \forall t$.
	What remains is to bound the last term above. Since $n_i(t-1)\geq C_i, \beta_{t-1}\leq B_i, \forall t\leq T$, this term is always upper bounded by
	\begin{equation}\label{eq:e-greedy-ub-2}
	\begin{aligned}
	\PP\left(\widehat{\mu}_i(t-1) + \frac{\beta_{t-1}}{n_i(t-1)}\geq \widehat{\mu}_{i^*}(t-1), n_i(t-1)\geq C_i\right)
	&\leq \PP\left(\widehat{\mu}_i(t-1) + \frac{B_i}{C_i}\geq \widehat{\mu}_{i^*}(t-1)\right)\\
	&=\PP\left(\widehat{\mu}_i(t-1) + \frac{\Delta_i}{3}\geq \widehat{\mu}_{i^*}(t-1)\right)
	\end{aligned}
	\end{equation}
	By union bound, we have $\PP\Big(\widehat{\mu}_i(t-1) + \frac{\Delta_i}{3}\geq \widehat{\mu}_{i^*}(t-1)\Big) \leq \PP\Big(\widehat{\mu}_i(t-1)\geq \mu_i+ \frac{\Delta_i}{3}\Big) + \PP\Big(\widehat{\mu}_{i^*}(t-1)\leq \mu_{i^*} -\frac{\Delta_k}{3}\Big)$. Based on Lemma~\ref{lem:e-greedy-concentration}, we have
	\begin{eqnarray}\label{eq:e-greedy-ub-aux-1}
	\PP\Big(\widehat{\mu}_i(t-1) + \frac{\Delta_i}{3}\geq \widehat{\mu}_{i^*}(t-1)\Big)
	& \leq 2x_{t}\cdot e^{-x_{t}/5} + \frac{18\sigma^2}{\Delta_i^2}e^{-\Delta_k^2 \lfloor x_{t}\rfloor/18\sigma^2}
	\end{eqnarray}
	We observe the fact that $x_t \geq \lfloor cK\rfloor-K+\frac{c}{2}\ln \frac{t}{\lfloor cK\rfloor+1} > 5$. Given $x e^{-x/5} \leq ye^{-y/5}, \forall x\geq y\geq 5$ and $e^{-x} \leq e^{-y}, \forall x\geq y$, we have
	\begin{equation*}
	\begin{array}{c}
	x_t e^{-x_t/5} \leq \left(\lfloor cK \rfloor -K  + \frac{c}{2}\ln \frac{t}{\lfloor cK \rfloor+1}\right) e^{-\frac{c}{10}\ln\frac{t}{\lfloor cK \rfloor+1}} = \left(\lfloor cK \rfloor - K + \frac{c}{2}\ln \frac{t}{\lfloor cK \rfloor+1}\right)\cdot\left(\frac{\lfloor cK \rfloor+1}{t}\right)^{c/10}\\
	\frac{\sigma^2}{\Delta_i^2} e^{-\Delta_i^2 \lfloor x_{t-1}\rfloor/18\sigma^2} \leq \frac{\sigma^2}{\Delta_i^2} e^{-\Delta_i^2 c\ln\frac{t}{\lfloor cK\rfloor+1}/36\sigma^2}= \frac{\sigma^2}{\Delta_i^2}\left(\frac{ \lfloor cK\rfloor+1}{t}\right)^{c\Delta_i^2/36\sigma^2}
	\end{array}
	\end{equation*}
	Combining the above inequalities and Fact~\ref{fact:harmonic-seq}, we can bound
	\begin{equation}\label{eq:e-greedy-ub-3}
	\begin{aligned}
	&\sum_{t=\lfloor cK\rfloor+1}^T 2x_{t} \cdot e^{-x_{t}/5} + \frac{18\sigma^2}{\Delta_i^2} e^{-\Delta_i^2 \lfloor x_{t}\rfloor/18\sigma^2}\\
	&\leq\sum_{t=\lfloor cK\rfloor+1}^T \left(2\lfloor cK\rfloor - 2K + c\ln\left(\frac{t}{\lfloor cK\rfloor+1}\right)\right)\cdot\left(\frac{\lfloor cK\rfloor+1}{t}\right)^{2} + \frac{18\sigma^2}{\Delta_i^2}\frac{\lfloor cK\rfloor+1}{t}\\
	&\leq (\lfloor cK\rfloor - K)\cdot\frac{2(\lfloor cK\rfloor+1)^2\pi^2}{3} + \left(c+\frac{18\sigma^2}{\Delta_i^2}\right)\sum_{t=\lfloor cK\rfloor+1}^T \frac{\lfloor cK\rfloor+1}{t}\\
	&\leq (\lfloor cK\rfloor - K)\cdot\frac{2(\lfloor cK\rfloor+1)^2\pi^2}{3} +(\lfloor cK\rfloor+1)\left(c+\frac{18\sigma^2}{\Delta_i^2}\right)\ln\frac{T}{\lfloor cK\rfloor}
	\end{aligned}
	\end{equation}
	The first inequality in the above holds because $c\geq\max\{20, \frac{36\sigma^2}{\Delta_i^2}\}$, and the second inequality is based on the fact that $\ln x < x, \forall x > 1$ and $\sum_{t=1}^T \frac{1}{t^2}\leq \frac{\pi^2}{3}$. The last inequality is the implication of Fact~\ref{fact:harmonic-seq}.  Moreover, utilizing Fact~\ref{fact:harmonic-seq}, we bound $\sum_{t=K+1}^T \frac{\eps_t}{K}$ in the following way,
	\begin{eqnarray}\label{eq:sum-epsilon}\sum_{t=K+1}^{T}\frac{\epsilon_t}{K} = \sum_{t=K+1}^{\lfloor cK\rfloor}\frac{1}{K} +\sum_{t=\lfloor cK\rfloor+1}^T\frac{\epsilon_t}{K}\leq \frac{\lfloor cK\rfloor-K}{K}+c\ln\frac{T}{\lfloor cK\rfloor},\end{eqnarray}
	Combining Equations~(\ref{eq:e-greedy-ub-1}), ~(\ref{eq:e-greedy-ub-2}), ~(\ref{eq:e-greedy-ub-aux-1}) and (\ref{eq:sum-epsilon}), we complete the proof.
\end{proof}

\subsection{Proof of Lemma~\ref{lem:ts-ub}}
We bound the terms in the decomposition of $\E[n_i(T)]$ in Eq.~\eqref{eq:ts-decomposition} using Lemma~\ref{lem:ts-ub-1} -- Lemma~\ref{lem:ts-ub-3}.

\begin{lemma}[Lemma 2.16 in \cite{AG2017}]\label{lem:ts-ub-1}
	Let $x_i = \mu_i + \frac{\Delta_i}{3}$ and $y_i = \mu_{i^*} -\frac{\Delta_i}{3}$,
	\begin{eqnarray*}
		\E\left[\sum_{t=K+1}^T \1\big\{I_t=i,E^{\mu}_i(t),\overline{\E^\theta_i(t)}\big\}\right]\leq \frac{18\ln T}{\Delta_i^2} + 1
	\end{eqnarray*}
\end{lemma}

\begin{lemma}[Eq.~(4) in \cite{AG2017}]\label{lem:ts-ub-2}
	$
	\sum_{t=K+1}^T \PP\left(I_t=i,E^{\mu}_i(t),\E^\theta_i(t)\right) \leq \sum_{s=K+1}^{T-1}\E\left[\frac{1}{p_{i,\tau_{i^*,s}+1}}-1\right]
	$
\end{lemma}

\begin{lemma}[Extension of Lemma 2.13 in \cite{AG2017}]\label{lem:ts-ub-geometric} 
Let $y_i = \mu_{i^*} - \frac{\Delta_i}{3}$,
$$
\E\left[\frac{1}{p_{i,\tau_{i^*,s}+1}} - 1\right]\leq \left\{
\begin{array}{cc}
e^{11/4\sigma^2} + \frac{\pi^2}{3}& \forall s\\
\frac{4}{T\Delta_i^2}& \text{if } s\geq \frac{72\ln(T\Delta_i^2)\cdot\max\{1, \sigma^2\}}{\Delta_i^2}
\end{array}
\right.
$$
\end{lemma}
\begin{proof}
This lemma extends Lemma 2.13 in \cite{AG2017} to our setting, and we mainly emphasize the required changes to the proof. Using the same notation as in \cite{AG2017}, let $\Theta_j$ denote the Gaussian random variable follows 
$\mathcal{N}(\widehat{\mu}_{i^*}(\tau_j+1), \frac{1}{j})$, given $\mathcal{F}_{\tau_j}$. Let $G_j$ be the geometric random variable denoting the number of consecutive independent trials until a sample of $\Theta_j$ becomes greater than $y_i$. Let $\gamma\geq 1$  be an integer and $z=2\sigma\sqrt{\ln \gamma}$. Then we have $\E\left[\frac{1}{p_{i,\tau_j+1}}-1\right]=\E[G_j]$. Following the same argument proposed in \cite{AG2017}, we have for any $\gamma > e^{11/4\sigma^2}$,
\begin{align*}
\PP(G_j<\gamma) \geq \left(1-\frac{1}{\gamma^2}\right)\PP\left(\widehat{\mu}_{i^*} + \frac{z}{\sqrt{j}} \geq y_i\right)
\end{align*}
For $n_{i^*}(t-1)=j$, $\mathcal{F}_{\tau_j}$, we have 
\begin{eqnarray*}
\PP \left(\widehat{\mu}_{i^*}(\tau_j+1) + \frac{z}{\sqrt{j}} \geq  y_i\right)  &\geq& \PP\left(\widehat{\mu}_{i^*}(\tau_j+1) + \frac{z}{\sqrt{j}} \geq \mu_{i^*}\right) \\
& \geq &  1 - e^{-\frac{z^2}{2\sigma^2}} \\
 & =& 1 - e^{-4\sigma^2\ln \gamma/2\sigma^2} = 1- \left(\frac{1}{\gamma}\right)^{2}
\end{eqnarray*}
Then $\PP\left(G_j<\gamma\right) \geq 1 - \frac{1}{\gamma^2} - \frac{1}{\gamma^{2}} = 1 -\frac{2}{t^2}$. Therefore,
\begin{align*}
\E[G_j] = \sum_{\gamma=0}^T \PP(G_j \geq \gamma) \leq e^{11/4\sigma^2} + \sum_{\gamma\geq 1} \frac{2}{t^2} \leq e^{11/4\sigma^2} + \frac{\pi^2}{3}
\end{align*}
By the proof of Lemma 2.13 in \cite{AG2017}, we have for any $D_i(T) \geq 0$,
\begin{align*}
\E\left[\frac{1}{p_{i,\tau_j+1}}-1\right]\leq \frac{1}{\left(1-\frac{1}{2}e^{-D_i(T)\Delta_i^2/72}\right)\left(1-e^{-D_i(T)\Delta_i^2/72\sigma^2}\right)}
\end{align*}
Since $D_i(T) = \frac{72\ln(T\Delta_i^2)\cdot\max\{1, \sigma^2\}}{\Delta_i^2}$, we have both $1-\frac{1}{2}e^{-D_i(T)\Delta_i^2/72}$ and $1-e^{-D_i(T)\Delta_i^2/72\sigma^2}$ are larger than or equal to $1-\frac{1}{T\Delta_i^2}$. Thus, $\E\left[\frac{1}{p_{i,\tau_j+1}}-1\right]$ can be bounded by $\frac{4}{T\Delta_i^2}$ when $j\geq D_{i, T}$.
\end{proof}

\begin{lemma}\label{lem:ts-ub-3}
	\begin{eqnarray}\label{eq:event-non-mu}
	\E\left[\sum_{t=K+1}^T \1\big\{I_t=i, \overline{E_i^\mu(t)}\big\}\right] \leq \max\Big\{\frac{6B_i}{\Delta_i}, \frac{144\sigma^2\ln T}{\Delta_i^2}\Big\} + 1
	\end{eqnarray}
\end{lemma}

\begin{proof}
Let $C_i(T) = \max\Big\{\frac{6B_i}{\Delta_i}, \frac{144\sigma^2\ln T}{\Delta_i^2}\Big\}$.
We first decompose the left hand side in Equation~(\ref{eq:event-non-mu}) as below,
\begin{equation}\label{eq:ts-ub-aux-3}
\begin{aligned}
\E\left[\sum_{t=K+1}^T \1\big\{I_t=i, \overline{E_i^\mu(t)}\big\}\right]& \leq \E\left[\sum_{t=K+1}^T \1\big\{I_t=i, \overline{E_i^\mu(t)}, n_i(t-1) \leq C_i(T)\big\}\right] \\
&+ \E\left[\sum_{t=K+1}^T \1\big\{I_t=i, \overline{E_i^\mu(t)}, n_i(t-1) \geq C_i(T)\big\}\right]
\end{aligned}
\end{equation}
The first term in the above decomposition is trivially bounded by $c_i(T)$. What remains is to bound second term 
\begin{align*}
& \E\left[\sum_{t=K+1}^T \1\big\{I_t=i, \overline{E_{i,t}^\mu}, n_i(t-1) 
 \geq  c_i(T)\big\}\right] \\
 \leq & \sum_{t=K+1}^T \PP\left(\overline{E_{i,t}^\mu}, n_i(t-1) \geq C_i(T)\right)\\
\leq & \sum_{t=K+1}^T\PP\left(\widehat{\mu}_{i,t-1} + \frac{\beta_{t-1}}{n_i(t-1)} \geq x_i\Big| n_i(t-1) \geq C_i(T)\right)\\
 \leq & \sum_{t=K+1}^T\PP\left(\widehat{\mu}_{i,t-1} + \frac{\beta_{t-1}}{n_i(t-1)} \geq x_i\Big| n_i(t-1) \geq C_i(T)\right)\\
\end{align*}
By union bound, we have 
\begin{align*}
& \PP\left(\widehat{\mu}_{i,t-1} + \frac{\beta_{t-1}}{n_i(t-1)} \geq x_i\Big| n_i(t-1) \geq C_i(T)\right) \\
\leq& \sum_{s=c_i(T)}^{t-1}\PP\left(\widehat{\mu}_{i,t-1} + \frac{B_i}{n_i(t-1)} \geq x_i\Big| n_i(t-1)=s\right)\\
 \leq & \sum_{s=c_i(T)}^{t-1} e^{-\frac{s \cdot \left(x_i - \mu_i - \frac{B_i}{s}\right)^2}{2\sigma^2}} \leq \sum_{s=1}^{t-1} \frac{1}{T^2} 
\end{align*}
The last inequality above uses Fact~(\ref{fact:sub-gaussian}) and the fact $s\geq c_i(T) \geq \frac{6B_i}{\Delta_i}$ and $s \geq \frac{144\sigma^2\ln T}{\Delta_i^2}$. Then the second term of the right hand side in Equations~\ref{eq:ts-ub-aux-3} can be bounded by $\sum_{t=K+1}^T \sum_{s=1}^{t-1}\frac{1}{T^2}\leq 1$.
\end{proof}

\subsection{Proof of Proposition~\ref{prop:ts-lb}}\label{app:ts-lower-bound}

We complete the proofs for \egreedy{} principal and Thompson Sampling separately. Similar to UCB principal, we derive the upper bound of $\E[n_{i^*}(T)]$ when all strategic arms use \lsi manipulation strategy, shown in Lemma~\ref{lem:e-greedy-lb} (for \egreedy{} principal) and Theorem~\ref{thm:ts-lb-opt-arm-ub} (for Thompson Sampling). Then Proposition~\ref{prop:ts-lb} is straightforward.

{\bf $\eps$-Greedy principal.}

\begin{lemma}\label{lem:e-greedy-lb}
	$\forall t > K$, let $\epsilon_t = \min\{1, \frac{cK}{t}\}$, where  a constant $c=\max\Big\{20, \frac{16\sigma^2}{\Delta_k^2},\forall k\in[K]\Big\}$, $B_i$ be the total budget for strategic arm. The expected number of plays of arm $i^*$ up to time $T$, if all strategic arms use $\mathtt{LSI}$, is bounded by                                                                         
	\begin{align*}
	\E\left[n_{i^*}(T)\right] \leq T - \sum_{i\neq i^*} \frac{B_i}{2\Delta_i}  + \mathcal{O}\left(\frac{\ln T}{\underline{\Delta}^2}\right)
	\end{align*}
\end{lemma}

\begin{proof}	
	Let $C_i = \frac{B_i}{2\Delta_i}, x_t = \frac{1}{2K}\sum_{s=K+1}^t\epsilon_s$ and for $t\geq \lfloor cK\rfloor + 1$, by Equation~\eqref{eq:e-greedy-x_t} $x_t \geq \lfloor cK\rfloor - K +  \frac{c}{2}\ln\frac{t}{\lfloor cK\rfloor+1}$.
	
	We first bound the probability of $\PP\left(\widehat{\mu}_{i^*}(t-1) \geq \widetilde{\mu}_i(t-1)\Big|n_i(t-1)\leq C_i\right)$ for $t \geq K+1$,
	\begin{equation}\label{eq:e-greedy-lb-1}
	\begin{aligned}
	&\PP\Big(\widehat{\mu}_{i^*}(t-1)\geq \widetilde{\mu}_i(t-1), n_i(t-1)\leq C_i\Big) \\
	=~ &~  \PP\Big(\widehat{\mu}_{i^*}(t-1)\geq \widehat{\mu}_i(t-1)+\frac{B_i}{n_i(t-1)}, n_i(t-1)\leq C_i\Big)\\
	\leq~&~\PP\Big(\widehat{\mu}_{i^*}(t-1)\geq \widehat{\mu}_i(t-1)+2\Delta_i\Big)\\
	\leq~&~ \PP\Big(\widehat{\mu}_{i^*}(t-1)\geq \mu_{i^*}+\frac{\Delta_i}{2}\Big) + \PP\Big(\widehat{\mu}_i(t-1)\leq \mu_{i}-\frac{\Delta_i}{2}\Big)\\
	\leq~&~ 2x_{t} \cdot e^{-x_{t}/5} + \frac{8\sigma^2}{\Delta_i^2} e^{-\Delta_i^2 \lfloor x_{t}\rfloor/8\sigma^2} \indent (\text{By Lemma~\ref{lem:e-greedy-concentration}})
	\end{aligned}
	\end{equation}
	We can decompose the expected number of plays of the optimal arm $i$, $\E[n_{i^*,T}]$, as follows,
	\begin{equation}\label{eq:e-greedy-lb-2}
	\begin{aligned}
	\E\left[n_{i^*}(T)\right]
	= 1 &+ \E\left[\sum_{t=K+1}^T \1\{I_t = i^*, \forall i\neq i^*n_i(t-1)\geq C_i\}\right] \\
	&+ \E\left[\sum_{t=K+1}^T \1\{I_t = i^*, \exists i\neq i^*, n_i(t-1)\leq C_i\}\right] \\
	\end{aligned}
	\end{equation}
	The first term in the above decomposition can be bounded by $T- \sum_{i\neq i^*}C_i$. This is because
	\begin{align*}
	&\E\left[\sum_{t=K+1}^T\1\big\{I_t=i^*, \forall i\neq i^*, n_i(t-1)\geq C_i\big\}\right]\\
	&\hspace{1cm}\leq \E\left[\sum_{t=K+1}^T\1\big\{I_t=i^*, n_{i^*}(t-1)\leq T-\sum_{i\neq i^*}C_i\big\}\right]\leq T-\sum_{i\neq i^*}C_i.
	\end{align*}
	By union bound, the second term is bounded by $\sum_{i\neq i^*} \E\left[\sum_{t=K+1}^T \1\{I_t = i^*, n_i(t-1)\leq C_i\}\right]$. Then, we bound the above summand using Equations~(\ref{eq:e-greedy-lb-1}) and the fact that $1-\eps_t=0$ when $t\leq \lfloor cK\rfloor$,
	\begin{equation}
	\begin{aligned}
	& \E\left[\sum_{t=K+1}^T \1\{I_t = i^*, n_i(t-1)\leq C_i\}\right] \\
	\leq & \sum_{t=K+1}^T \frac{\epsilon_t}{K}+ \sum_{t=K+1}^T (1-\epsilon_t)\cdot \PP\Big(\widehat{\mu}_{i^*}(t-1)\geq \widetilde{\mu}_i(t-1), n_i(t-1)\leq C_i\Big)\\
	\leq & \sum_{t=K+1}^T \frac{\epsilon_t}{K} +\sum_{t=\lfloor cK\rfloor+1}^T 2x_{t} \cdot e^{-x_{t}/5} + \frac{8\sigma^2}{\Delta_i^2} e^{-\Delta_i^2 \lfloor x_{t}\rfloor/8\sigma^2}
	\end{aligned}
	\end{equation}
	What remains is to bound the last term in the above equations.
	Following the same arguments and proof procedure in Equations~\eqref{eq:e-greedy-ub-3}, we can bound
	\begin{equation}\label{eq:e-greedy-lb-3}
	\begin{aligned}
	&\sum_{t=\lfloor cK\rfloor+1}^T 2x_{t} \cdot e^{-x_{t}/5} + \frac{8\sigma^2}{\Delta_i^2} e^{-\Delta_i^2 \lfloor x_{t}\rfloor/8\sigma^2}\\
	&\leq (\lfloor cK\rfloor - K)\cdot\frac{2(\lfloor cK\rfloor+1)^2\pi^2}{3} +(\lfloor cK\rfloor+1)\left(c+\frac{8\sigma^2}{\Delta_i^2}\right)\ln\frac{T}{\lfloor cK\rfloor}
	\end{aligned}
	\end{equation}

	By Eq.~\eqref{eq:sum-epsilon}, we have
	\begin{align*}
	\E[n_{i^*}(T)]&\leq T - \sum_{i\neq i^*} \frac{B_i}{2\Delta_i} + \frac{\lfloor cK\rfloor}{K}+c\ln\frac{T}{\lfloor cK\rfloor} \\
	& + \sum_{i\neq i^*}\left((\lfloor cK\rfloor - K)\cdot\frac{2(\lfloor cK\rfloor+1)^2\pi^2}{3} +(\lfloor cK\rfloor+1)\left(c+\frac{8\sigma^2}{\Delta_i^2}\right)\ln\frac{T}{\lfloor cK\rfloor}\right)\\
	& \leq T - \sum_{i\neq i^*} \frac{B_i}{2\Delta_i}  + \mathcal{O}\left(\frac{\ln T}{\underline{\Delta}^2}\right)
	\end{align*}
\end{proof}

{\bf Thompson Sampling principal.}
Here we slightly abuse notations, and use $E^\mu_{i^*}(t)$ to denote the event that $\widehat{\mu}_{i^*}(t-1) \leq v_i$ whereas $E^\theta_{i^*}(t)$ to denote the event that $\theta_{i^*}(t) \leq w_i$, where $\mu_{i^*} < v_i < w_i$.

\begin{theorem}\label{thm:ts-lb-opt-arm-ub}
	\begin{align*}
	\E[n_{i^*}(T)] \leq T - \sum_{i\neq i^*}\frac{B_i}{2\Delta_i} + \mathcal{O}\left(\frac{\ln T}{\underline{\Delta}^2}\right)
	\end{align*}
\end{theorem}
\begin{proof}
	We decompose the expected number of plays of the optimal arm $i^*$ as follows,
	\begin{align*}
	\E[n_{i^*}(T)]\leq 1 &+ \sum_{t=K+1}^T \PP\left(I_t =i^*, \overline{E_{i^*}^\mu(t)}\right) + \sum_{t=K+1}^T\PP\left(I_t = i^*, \overline{E^{\theta}_{i^*}(t)}, E^\mu_{i^*}(t)\right) \\
	&+ \sum_{t=K+1}^T\PP\left(I_t = i^*, E^{\theta}_{i^*}(t), E^\mu_{i^*}(t)\right) 
	\end{align*}
	Then we bound each of the above terms. Lemma~\ref{lem:ts-lb-1},~\ref{lem:ts-lb-2} and ~\ref{lem:ts-lb-3} show the upper bound of each term and complete the proof.
\end{proof}

\begin{lemma}\label{lem:ts-lb-1}
Let $v_i = \mu_{i^*} + \frac{\Delta_i}{3}$,
\begin{align*}
\sum_{t=K+1}^T \PP\left(I_t =i^*, \overline{E_{i^*}^\mu(t)}\right) \leq \frac{18\sigma^2}{\Delta_i^2}
\end{align*}
\end{lemma}

\begin{proof}
Following the proof of Lemma 2.11 in \cite{AG2017}, we have
\begin{align*}
\sum_{t=K+1}^T \PP\left(I_t =i^*, \overline{E_{i^*}^\mu(t)}\right)& \leq \sum_{s=1}^{T-1}\PP\left(\overline{E_{i^*}^\mu(\tau_{i^*,s+1})}\right) = \sum_{s=1}^{T-1}\PP\left(\widehat{\mu}_{i^*}(\tau_{i^*,s+1})>v_i\right)\\
&\leq \sum_{s=1}^{T-1}\exp\left(-\frac{s(v_i-\mu_{i^*})^2}{2\sigma^2}\right) \leq \frac{2\sigma^2}{(v_i - \mu_{i^*})^2}
\end{align*}
The first inequality holds because each summand on the right hand side in this inequality is a fixed number since the distribution of $\widehat{\mu}_{i^*}(\tau_{i^*,s+1})$ only depends on $s$. The second inequality is based on Fact~\ref{fact:gaussian-concentration} and the third inequality goes through because $\sum_{k=1}^{\infty}e^{-kx} \leq \frac{1}{x}, \forall x > 0$.
\end{proof}

Notice that Lemma~\ref{lem:ts-ub-2} holds \emph{independently with} the identity of the arm. Then the following Lemma can be directly implied.
\begin{lemma}\label{lem:ts-lb-2}
	Let $v_i = \mu_{i^*} +\frac{\Delta_i}{3}$ and $w_i = \mu_{i^*} +\frac{2\Delta_i}{3}$
	\begin{eqnarray*}
		\sum_{t=K+1}^T\PP\left(I_t = i^*, \overline{E^{\theta}_{i^*}(t)}, E^\mu_{i^*}(t)\right) \leq \frac{18\ln T}{\Delta_i^2} + 1
	\end{eqnarray*}
\end{lemma}

\begin{proof}
The proof of Lemma 2.16 in \cite{AG2017} can be directly applied here by regarding arm $i^*$ as a standard sub-optimal arm $i$.
\end{proof}

What remains is to bound $\sum_{t=K+1}^T\PP\left(I_t =i^*, E_{i^*}^\theta(t),E_{i^*}^\mu(t)\right)$. To this end, we show some auxiliary lemmas in the following. Lemma~\ref{lem:ts-lb-aux-1} mimics Lemma 2.8 in \cite{AG2017}, which bridges the probability that arm $i^*$ will be pulled and the probability that arm $i$ will be pulled at time $t$. Lemma~\ref{lem:ts-lb-aux-2} bounds the term $\E\left[\frac{1}{q_{i,\tau_{i,s}+1}}-1\right]$ by a reduction to the case shown in Lemma~\ref{lem:ts-ub-2}.

\begin{lemma}\label{lem:ts-lb-aux-1}
	For any instantiation $F_{t-1}$ of $\mathcal{F}_{t-1}$, let $q_{i,t} := \PP\left(\theta_i(t)>w_i\Big|F_{t-1}\right)$, we have
	\begin{eqnarray*}
		\PP\left(I_t=i^*,E_{i^*}^\theta(t),E_{i^*}^\mu(t)\Big|F_{t-1}\right) \leq \frac{1-q_{i,t}}{q_{i,t}}\PP\left(I_t=i, E_{i^*}^\theta(t),E_{i^*}^\mu(t)\Big|F_{t-1}\right)
	\end{eqnarray*}
\end{lemma}

\begin{proof}
Since $E^\mu_{i^*}(t)$ is only determined by the instantiation $F_{t-1}$ of $\mathcal{F}_{t-1}$, we can assume event 	$E^\mu_{i^*}(t)$ is true without loss of generality. Then, it is sufficient to show that for any $F_{t-1}$ we have
\begin{eqnarray*}
	\PP\left(I_t = i^*\Big|E^\theta_{i^*}(t), F_{t-1}\right) \leq \frac{1-q_{i,t}}{q_{i,t}}\PP\left(I_t=i, \Big|E_{i^*}^\theta(t),F_{t-1}\right)
\end{eqnarray*}
Note, given $E_{i^*}^\theta(t)$, $I_t=i^*$ implies $\theta_j(t)\leq w_i, \forall j$, meanwhile, $\theta_i(t)$ is independent with $\theta_j(t), j\neq i$, given $\mathcal{F}_{t-1}=F_{t-1}$. Therefore, we have
\begin{align*}
\PP\left(I_t = i^*\Big|E^\theta_{i^*}(t), F_{t-1}\right)&\leq \PP\left(\theta_j(t)\leq w_i,\forall j\Big|E^\theta_{i^*}(t), F_{t-1}\right)\\
& = \PP\left(\theta_i(t)\leq w_i\Big|F_{t-1}\right)\cdot \PP\left(\theta_j(t)\leq w_i, \forall j\neq i\Big|E^\theta_{i^*}(t), F_{t-1}\right)
\end{align*}
On the other side, 
\begin{align*}
\PP\left(I_t = i\Big|E^\theta_{i^*}(t), F_{t-1}\right)&\geq \PP\left(\theta_i(t)> w_i\geq\theta_j(t),\forall j\neq i\Big|E^\theta_{i^*}(t), F_{t-1}\right)\\
& = \PP\left(\theta_i(t)> w_i\Big|F_{t-1}\right)\cdot \PP\left(\theta_j(t)\leq w_i, \forall j\neq i\Big|E^\theta_{i^*}(t), F_{t-1}\right)
\end{align*}
Thus, the above two inequalities implies the correctness of the Lemma.
\end{proof}

\begin{lemma}\label{lem:ts-lb-aux-2}
	Let $w_i=\mu_{i^*}+\frac{2\Delta_i}{3}$. For any $s\geq 1$, given $n_{i}(\tau_{i,s})\leq \frac{B_i}{2\Delta_i}$, we have
	$$
	\E\left[\frac{1}{q_{i,\tau_{i,s}+1}} - 1\Big|n_{i}(\tau_{i,s})\leq \frac{B_i}{2\Delta_i}\right]\leq \left\{
	\begin{array}{cc}
	e^{11/4\sigma^2} + \frac{\pi^2}{3}& \forall s \\
	\frac{1}{T\Delta_i}& \text{if } s\geq L_i(T)
	\end{array}
	\right.
	$$
	where $L_i(T) = \frac{72\ln(T\Delta_i^2)\cdot\max\{1, \sigma^2\}}{\Delta_i^2}$.
\end{lemma}
\begin{proof}
	We prove this Lemma by a reduction to Lemma~\ref{lem:ts-ub-geometric}.
	First, we observe $\theta_{i}(\tau_{i,s}+1)\sim \mathcal{N}\left(\widetilde{\mu}_i(\tau_{i,s}), \frac{1}{n_i(\tau_{i,s})}\right)$, where $\widetilde{\mu}_i(\tau_{i,s}) = \widehat{\mu}_i(\tau_{i,s}) + \frac{B_i}{n_i(\tau_{i,s})}$. Given $n_{i}(\tau_{i,s})\leq \frac{B_i}{\Delta_i}$, we have $\widetilde{\mu}_i(\tau_{i,s}) \geq \widehat{\mu}_i(\tau_{i,s}) + 2\Delta_i$. Let $\zeta_i(\tau_{i,s}+1)$ denote the random variable of Gaussian distribution $\mathcal{N}\left(\widehat{\mu}_i(\tau_{i,s}),\frac{1}{n_i(\tau_{i,s})}\right)$. By the fact that a Gaussian random variable $a \sim \mathcal{N}(m, \sigma^2)$ is stochastically dominated by any $b\sim \mathcal{N}(m',\sigma^2)$ when $m<m'$, we have for any $F_{t-1}$ of $\mathcal{F}_{t-1}$
	\begin{align*}
	q_{i,\tau_{i,s}+1}& = \PP\left(\theta_i(\tau_{i,s}+1)>w_i\Big|F_{t-1}\right)\geq \PP\left(\zeta_i(\tau_{i,s}+1)+2\Delta_i>w_i\Big|F_{t-1}\right)\\
	&= \PP\left(\zeta_i(\tau_{i,s}+1)>\mu_{i}-\frac{\Delta_i}{3}\Big|F_{t-1}\right) := \eta_{i, \tau_{i,s}+1}
	\end{align*}
	
	Therefore, $\E\left[\frac{1}{q_{i,\tau_{i,s}+1}} - 1\right] \leq \E\left[\frac{1}{\eta_{i,\tau_{i,s}+1}} - 1\right]$. Denote $u_i := \mu_{i} - \frac{\Delta_i}{3}$. Recall  $$p_{i,\tau_{i,s}+1} = \PP\left(\theta_{i^*}(\tau_{i^*,s}+1)>\mu_{i^*}-\frac{\Delta_i}{3}\Big|F_{t-1}\right),$$ we observe $\eta_{i,\tau_{i,s}+1}$ is analogous to $p_{i,\tau_{i,s}+1}$ in formula, when we replace $\mu_i$ and $\widehat{\mu}_{i}(\tau_{i,s}+1)$ by $\mu_{i^*}$ and $\widehat{\mu}_{i^*}(\tau_{i^*,s}+1)$ respectively (i.e. change arm $i$ by $i^*$). Recall the proof in Lemma~\ref{lem:ts-ub-2}, it only depends on the relationship between $y_i=\mu_{i^*}-\frac{\Delta_i}{3}$ and $\mu_{i^*}$, which is the same as $u_i$ and $\mu_i$ in $\eta_{i,\tau_{i,s}+1}$. Thus, the proof of Lemma~\ref{lem:ts-ub-2} can be directly applied here to bound $\E\left[\frac{1}{\eta_{i,\tau_{i,s}+1}} - 1\right]$. 
\end{proof}

\begin{lemma}\label{lem:ts-lb-3}
	\begin{align*}
	& \sum_{t=K+1}^T \PP\left(I_t =i^*, E_{i^*}^\theta(t),E_{i^*}^\mu(t)\right) \\
	 \leq ~ &~~ T-\sum_{i\neq i^*}\frac{B_i}{2\Delta_i} + \sum_{i\neq i^*} \left(\big(e^{11/4\sigma^2} + \frac{\pi^2}{3}\big)\cdot \frac{72\ln(T\Delta_i^2)\cdot\max\{1, \sigma^2\}}{\Delta_i^2} + \frac{4}{\Delta_i^2}\right)
	\end{align*}
\end{lemma}

\begin{proof}
We first decompose the target term by thresholding $n_i(t-1)$ as follows,
\begin{equation}\label{eq:ts-lb-1}
\begin{aligned}
& \sum_{t=K+1}^T \PP\left(I_t =i^*, E_{i^*}^\theta(t),E_{i^*}^\mu(t)\right)  \\
\leq~&~~ \E\left[\sum_{t=K+1}^T\1\Big\{I_t=i^*,E_{i^*}^\theta(t),E_{i^*}^\mu(t), \forall i\neq i^*, n_i(t-1)\geq \frac{B_i}{2\Delta_i}\Big\}\right]\\
&~~~~~+ \sum_{t=K+1}^T\PP\left(I_t=i^*,E_{i^*}^\theta(t),E_{i^*}^\mu(t), \exists i\neq i^*, n_i(t-1)\leq \frac{B_i}{2\Delta_i}\right)
\end{aligned}
\end{equation}

For the first term in above decomposition, it can be trivially upper bounded by $T-\sum_{i\neq i^*}\frac{B_i}{2\Delta_i}$. By union bound and Lemma~\ref{lem:ts-lb-aux-1}, we can bound the second term as follows,
\begin{align*}
&\sum_{t=K+1}^T\PP\left(I_t=i^*,E_{i^*}^\theta(t),E_{i^*}^\mu(t), \exists i\neq i^*, n_i(t-1)\leq \frac{B_i}{2\Delta_i}\right)\\
& \hspace{1cm} \leq \sum_{i\neq i^*} \sum_{t=K+1}^T\PP\left(I_t=i^*,E_{i^*}^\theta(t),E_{i^*}^\mu(t), \exists i\neq i^*, n_i(t-1)\leq \frac{B_i}{2\Delta_i}\right)\\
&\hspace{1cm}= \sum_{i\neq i^*}\sum_{t=K+1}^T\E\left[\PP\left(I_t=i^*,E_{i^*}^\theta(t),E_{i^*}^\mu(t), n_i(t-1)\leq \frac{B_i}{2\Delta_i}\Big|\mathcal{F}_{t-1}\right)\right]\\
&\hspace{1cm}= \sum_{i\neq i^*}\sum_{t=K+1}^T\E\left[\frac{1-q_{i,t}}{q_{i,t}}\cdot\PP\left(I_t=i,E_{i^*}^\theta(t),E_{i^*}^\mu(t), n_i(t-1)\leq \frac{B_i}{2\Delta_i}\Big|\mathcal{F}_{t-1}\right)\right]\\
&\hspace{1cm}\leq\sum_{i\neq i^*}\sum_{t=K+1}^T\E\left[\frac{1-q_{i,t}}{q_{i,t}}\cdot\PP\left(I_t=i,E_{i^*}^\theta(t),E_{i^*}^\mu(t)\Big| n_i(t-1)\leq \frac{B_i}{2\Delta_i}, \mathcal{F}_{t-1}\right)\right]\\
&\hspace{1cm}= \sum_{i\neq i^*}\sum_{t=K+1}^T\E\left[\frac{1-q_{i,t}}{q_{i,t}}\cdot\1\Big\{I_t=i,E_{i^*}^\theta(t),E_{i^*}^\mu(t)\Big\}\Big| n_i(t-1)\leq \frac{B_i}{2\Delta_i}\right]\\
\end{align*}
Observe that $q_{i,t} =\PP\left(\theta_i(t)>w_i\Big|\mathcal{F}_{t-1}\right)$ changes only at the time step after each pull of arm $i$. Therefore we can bound the above term by,
\begin{align*}
& \sum_{s=1}^{T-1}\E\left[\frac{1-q_{i,\tau_{i,s}+1}}{q_{i,\tau_{i,s}+1}}\cdot \sum_{t=\tau_{i,s}+1}^{\tau_{i,s+1}}\1\big\{I_t=i, E_{i^*}^\theta(t),E_{i^*}^\mu(t)\big\}\Big|  n_i(\tau_{i,s})\leq \frac{B_i}{2\Delta_i}\right] \\
\leq ~& ~~ \sum_{s=1}^{T-1}\E\left[\frac{1-q_{i,\tau_{i,s}+1}}{q_{i,\tau_{i,s}+1}}\Big| n_i(\tau_{i,s})\leq \frac{B_i}{2\Delta_i}\right] 
\end{align*}
Combining Lemma~\ref{lem:ts-lb-aux-2} and Equation~(\ref{eq:ts-lb-1}), we complete the proof.
\end{proof}

\section{Additional Simulations}\label{app:simulation}
We report our simulation results for bounded rewards in this section. Similarly, we also consider a stochastic bandit setting with three arms. The reward of each arm lies within the interval $[0,1]$. The distributions of rewards of each arm are $\mathtt{Beta}(1,1)$, $\mathtt{Beta}(2,1)$ and $\mathtt{Beta}(3,1)$ respectively. In $\vareps$-Greedy algorithm, we use a different $\vareps_t$ parameter, i.e. $\vareps_t = \min\{1, \frac{20}{t}\}$. We run simulations for the same settings as those in Section~\ref{sec:simulation} and report the results in Figure~\ref{fig:bounded-rgt-time} and~\ref{fig:bounded-rgt-budget}. These figures illustrate similar performances for bounded rewards as for unbounded rewards.

\begin{figure}[ht]
\begin{subfigure}{0.33\textwidth}
\centering
\includegraphics[scale=0.45]{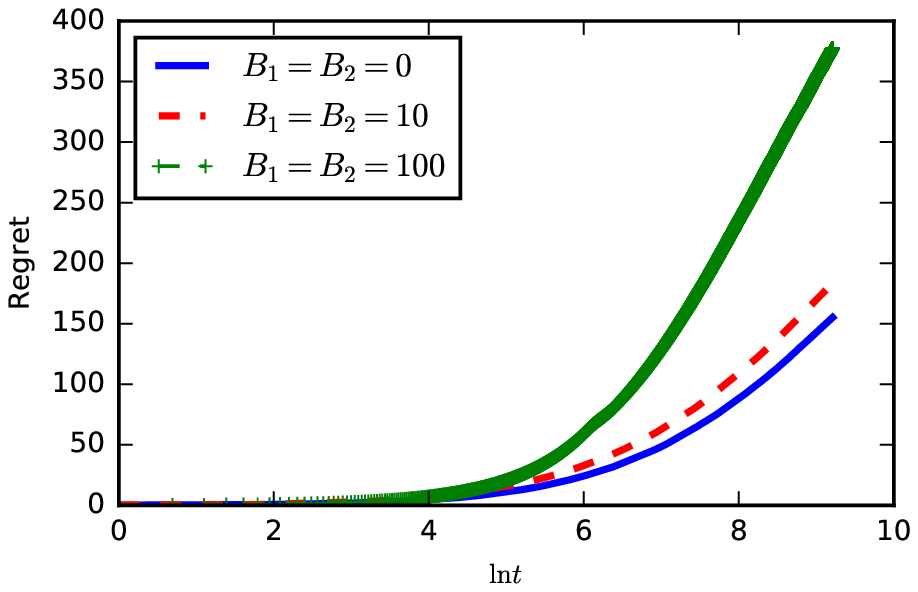}
\end{subfigure}
\begin{subfigure}{0.33\textwidth}
\centering
\includegraphics[scale=0.45]{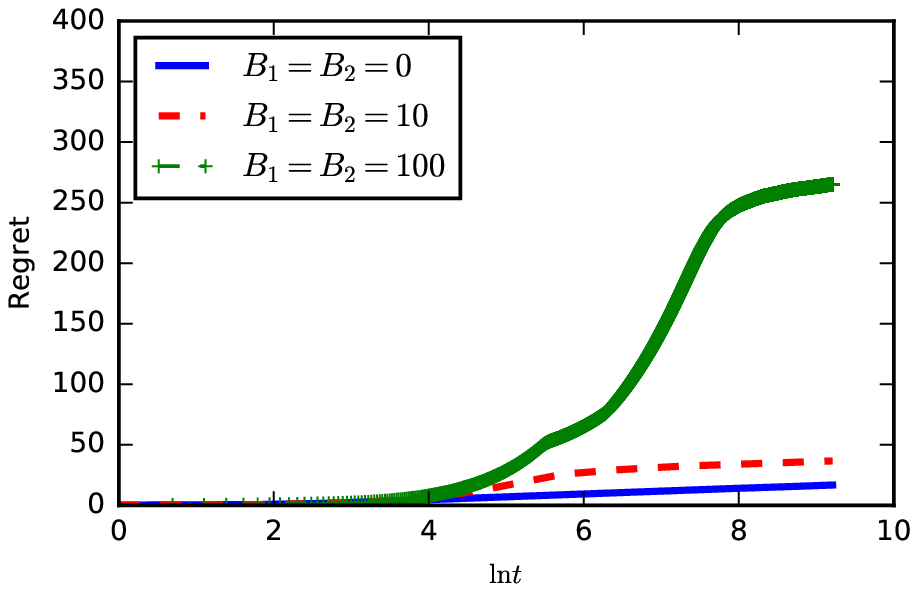}
\end{subfigure}
\begin{subfigure}{0.33\textwidth}
\centering
\includegraphics[scale=0.45]{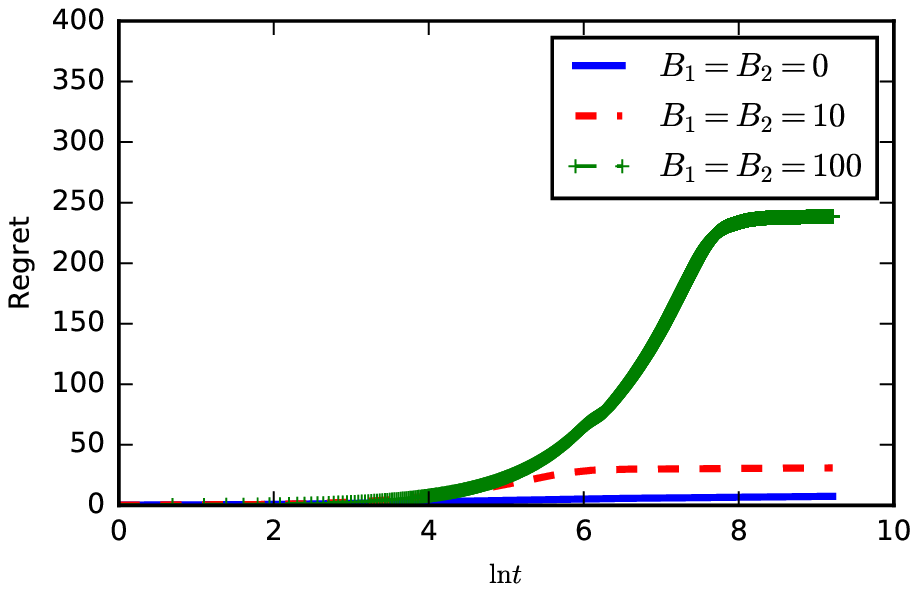}
\end{subfigure}
\caption{$[0, 1]$ bounded rewards: plots of regret with $\ln t$ for UCB principal (left), $\vareps$-Greedy principal (middle), and Thompson Sampling principal (right), as $B_1$ and $B_2$ vary. We set $B_3 = 0$ for the three algorithms.}
\label{fig:bounded-rgt-time}
\end{figure}
\begin{figure}[ht]
\begin{subfigure}{0.33\textwidth}
\centering
\includegraphics[scale=0.45]{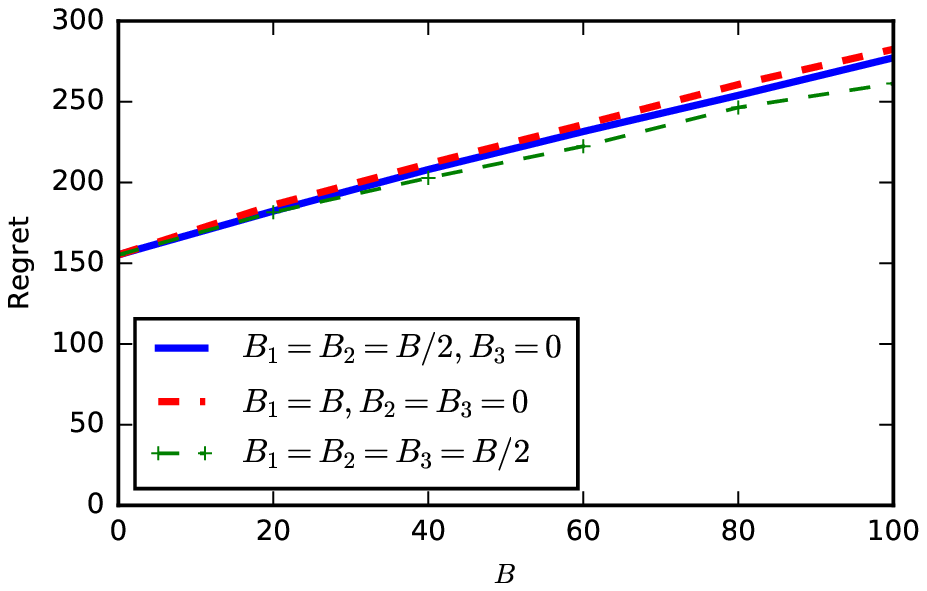}
\end{subfigure}
\begin{subfigure}{0.33\textwidth}
\centering
\includegraphics[scale=0.45]{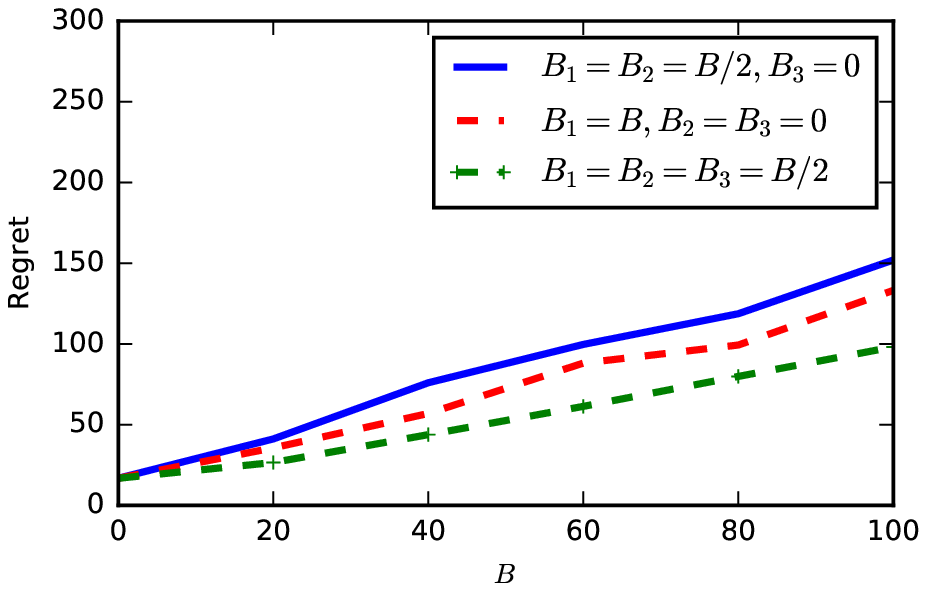}
\end{subfigure}
\begin{subfigure}{0.33\textwidth}
\centering
\includegraphics[scale=0.45]{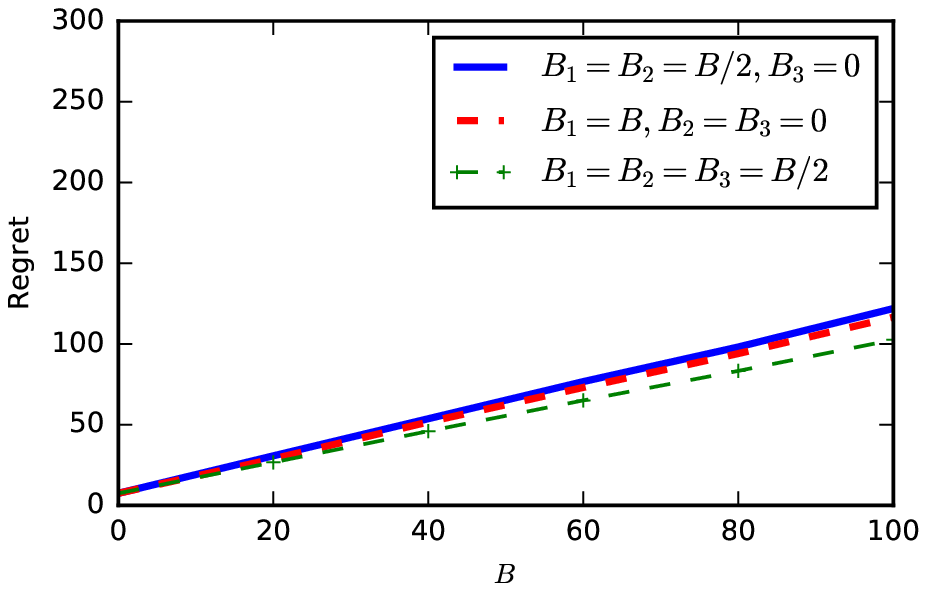}
\end{subfigure}
\caption{$[0, 1]$ bounded rewards: plots of regret with total budget $B$ of strategic arms (arm 1 and 2) for UCB principal (left), $\vareps$-Greedy principal (middle), and Thompson Sampling principal (right), as $B_i$ varies.}
\label{fig:bounded-rgt-budget}
\vspace{-10pt}
\end{figure}

\end{document}